\documentclass{article}

\usepackage[final,nonatbib]{neurips_2021}

\usepackage[utf8]{inputenc} 
\usepackage[T1]{fontenc}    
\usepackage{hyperref}       
\usepackage{url}            
\usepackage{booktabs}       
\usepackage{amsfonts}       
\usepackage{nicefrac}       
\usepackage{microtype}      
\usepackage{xcolor}         
\usepackage[pdftex]{graphicx}
\usepackage{amsmath}

\usepackage{svg}
\usepackage{amsthm}
\usepackage{amsfonts}

\usepackage{graphicx}
\usepackage{subfig}
\usepackage{multicol}
\usepackage{multirow}
\usepackage{mathtools} 
\usepackage{amssymb} 
\usepackage{systeme} 
\usepackage[normalem]{ulem}

\usepackage{wrapfig}

\usepackage{bbm}
\newcommand{\1}{\mathbbm{1}}
\newcommand{\C}{\mathcal{C}}
\newcommand{\Var}{\mathrm{Var}}
\newcommand{\Cov}{\mathrm{Cov}}

\DeclarePairedDelimiter\abs{\lvert}{\rvert}%
\DeclarePairedDelimiter\norm{\lVert}{\rVert}%
\makeatletter
\let\oldabs\abs
\def\abs{\@ifstar{\oldabs}{\oldabs*}}
\let\oldnorm\norm
\def\norm{\@ifstar{\oldnorm}{\oldnorm*}}
\makeatother

\newtheorem{theorem}{Theorem}
\newtheorem{definition}{Definition}
\newtheorem{statement}{Statement}
\newtheorem{lemma}{Lemma}

\usepackage{pifont}
\newcommand{\cmark}{\textcolor{green}{\ding{51}}}%
\newcommand{\xmark}{\textcolor{red}{\ding{55}}}

\usepackage{todonotes}

\title{Good Classification Measures and How to Find Them}

\author{
  Martijn G\"{o}sgens \\
  Eindhoven University of Technology \\ Eindhoven, The Netherlands \\
  \texttt{research@martijngosgens.nl} \\
  \And
  Anton Zhiyanov \\
  Yandex Research, HSE University \\
  Moscow, Russia \\
   \texttt{zhiyanovap@gmail.com} \\
  \AND
  Alexey Tikhonov \\
  Yandex \\
  Berlin, Germany \\
  \texttt{altsoph@gmail.com} \\
  \And
  Liudmila Prokhorenkova \\
  Yandex Research, HSE University, MIPT \\
  Moscow, Russia \\
  \texttt{ostroumova-la@yandex.ru} 
}

\begin{document}

\maketitle

\begin{abstract}
Several performance measures can be used for evaluating classification results: accuracy, F-measure, and many others. Can we say that some of them are better than others, or, ideally, choose one measure that is best in all situations? To answer this question, we conduct a systematic analysis of classification performance measures: we formally define a list of desirable properties and theoretically analyze which measures satisfy which properties. We also prove an impossibility theorem: some desirable properties cannot be simultaneously satisfied. Finally, we propose a new family of measures satisfying all desirable properties except one. This family includes the Matthews Correlation Coefficient and a so-called Symmetric Balanced Accuracy that was not previously used in classification literature. We believe that our systematic approach gives an important tool to practitioners for adequately evaluating classification results.
\end{abstract}

\section{Introduction}

\emph{Classification} is a classic machine learning task that is used in countless applications. To evaluate classification results, one has to compare the predicted labeling of a given set of elements with the actual (true) labeling. For this, \emph{performance measures} are used, and there are many well-known ones like accuracy, F-measure, and so on~\cite{hossin2015review,japkowicz2011evaluating}. The fact that different measures behave differently is known throughout the literature~\cite{ferri2009experimental,hossin2015review,luque2019impact,powers2011evaluation}. For instance, accuracy is known to be biased towards the majority class. Thus, different measures may lead to different evaluation results, and it is important to choose an appropriate measure. While there are attempts to compare performance measures and describe their properties~\cite{chicco2020advantages,cortes2004auc,hossin2015review,huang2005using,sebastiani2015axiomatically,sokolova2009systematic}, the problem still lacks a systematic approach, and our paper aims at filling this gap.\footnote{We describe related research in detail in Appendix~\ref{sec:related}.} Our research is particularly motivated by a recent paper~\cite{gosgens2019systematic} providing a systematic analysis of evaluation measures for the \emph{clustering} task. We transfer many proposed properties to the classification problem and extend the research by adding more properties, new measures, and novel theoretical results.

To provide a systematic comparison of performance measures, we formally define a list of properties that are desirable across various classification tasks. The proposed properties can be applied both to binary and multiclass problems. Some properties are intuitive and straightforward, like symmetry, while others are more tricky. A particularly important property is called \emph{constant baseline}. It requires a measure not to be biased towards particular predicted class sizes. For each measure and each property, we formally prove or disprove that the property is satisfied. We believe that this analysis is essential for better understanding the differences between the performance measures.

Then, we analyze relations between different properties in the binary case and prove an impossibility theorem: it is impossible for a performance measure to be linearly transformable to a metric and simultaneously have the constant baseline property. This means that at least one of these properties has to be discarded. If we relax the set of properties by discarding the distance requirement, the remaining ones can be simultaneously satisfied. In fact, we propose a family of measures called Generalized Means (GM), satisfying all the properties except distance and generalizing the well-known Matthews Correlation Coefficient (CC). In addition to CC, this class also contains another intuitive measure that we name \emph{Symmetric Balanced Accuracy}. To the best of our knowledge, this measure has not been previously used for classification evaluation.\footnote{For clustering evaluation, there is an analog known as \emph{Sokal\&Sneath's measure}~\cite{gosgens2019systematic}.} If we instead discard the constant baseline (but keep its approximation), then the arccosine of CC is a measure satisfying all the properties.

We also demonstrate through a series of experiments that different performance measures can be inconsistent in various situations. We notice that measures having more desirable properties are usually more consistent with each other.

We hope that our research will motivate further studies analyzing the properties of performance measures for classification and other problems since there are still plenty of questions to be answered. 

\section{Performance measures for classification}\label{sec:indices}

In this section, we define measures that are commonly used for evaluating classification results. Classification problems can be divided into binary, multiclass, and multilabel. In this paper, we focus on \emph{binary} and \emph{multiclass} and leave multilabel for future research. There are several types of performance measures: \emph{threshold} measures assume that predicted labels deterministically assign each element to a class (e.g., accuracy); \emph{probability} measures assume that the predicted labels are soft and compare these probabilities with the actual outcome (e.g., the cross-entropy loss); \emph{ranking} measures take into account the relative order of the predicted soft labels, i.e., quantify whether the elements belonging to a class have higher predicted probabilities compared to other elements (e.g., area under the ROC curve or average precision). Our research focuses on threshold measures.

Now we introduce notation needed to formally define binary and multiclass threshold measures.\footnote{For convenience, we list the notation used in the paper in Table~\ref{tab:notation} in Appendix.} Let $n>0$ be the number of elements in the dataset and let $m\ge2$ denote the number of classes. We assume that there is \emph{true labeling} classifying elements into $m$ classes and also \emph{predicted labeling}.  Let $\mathcal{C}$ be the confusion matrix: each matrix element $c_{ij}$ denotes the number of elements with true label $i$ and predicted label $j$. For binary classification, $c_{11}$ is \emph{true positive} (TP), $c_{00}$ is \emph{true negative} (TN), $c_{10}$ is \emph{false negative} (FN), and $c_{01}$ is \emph{false positive} (FP). We use the notation $a_{i} = \sum_{j=0}^{m-1} c_{ij}, b_i = \sum_{j=0}^{m-1} c_{ji}$ for the sizes of $i$-th class in the true and predicted labelings, respectively. Finally, we denote classification measures by $M(\mathcal{C})$ or $M(A,B)$, where $A$ and $B$ are true and predicted labelings, and write $M(c_{11},c_{10},c_{01},c_{00})$ for binary ones.

\begin{table}
\caption{Commonly used (above the line) and novel (below the line) validation measures}
\label{tab:validation_indices}
\vspace{3pt}
\centering
\begin{small}
\begin{tabular}{lcc}
\toprule
 & Binary & Multiclass \\
\midrule
\vspace{3pt}
F-measure ($F_{\beta}$) & $\frac{(1+\beta^2) \cdot c_{11}}{(1+\beta^2) \cdot c_{11} + \beta^2 \cdot c_{10} + c_{01} }$ & micro / macro / weighted \\
\vspace{3pt}
Jaccard (J) & $\frac{c_{11}}{c_{11} + c_{10} + c_{01}}$ & micro / macro / weighted \\
\vspace{3pt}
Matthews Coefficient (CC) & $\frac{c_{11} c_{00} - c_{01} c_{10}}{\sqrt{b_1 \cdot a_1 \cdot b_0 \cdot a_0}}$ & 
$\frac{n \sum_{i=0}^{m-1} c_{ii} - \sum_{i=0}^{m-1} b_i a_{i}}{\sqrt{\left(n^2 - \sum_{i=0}^{m-1} b_i^2\right)\left(n^2 - \sum_{i=0}^{m-1} a_i^2\right)}}$ \\
\vspace{3pt}
Accuracy (Acc) & 
\multicolumn{2}{c}{$\frac{\sum_{i=0}^{m-1} c_{ii}}{n}$}
\\
\vspace{3pt}
Balanced Accuracy (BA) &
\multicolumn{2}{c}{$\frac{1}{m}\sum_{i=0}^{m-1} \frac{c_{ii}}{a_{i}}$} 
\\
\vspace{3pt}
Cohen's Kappa ($\kappa$) & 
\multicolumn{2}{c}{$\frac{n\sum_{i=0}^{m-1} c_{ii} - \sum_{i=0}^{m-1} a_{i} b_{i} }{n^2 - \sum_{i=0}^{m-1} a_{i} b_{i}}$} \\
\vspace{3pt}
Confusion Entropy (CE) &  \multicolumn{2}{c}{ $ -\frac{1}{2 n} \sum\limits_{i,j: i \neq j} \left(c_{ji} \log_{2m-2}{\frac{c_{ji}}{a_j + b_j}} + c_{ij} \log_{2m-2}{\frac{c_{ij}}{a_j+b_j}}\right)$}\\
\midrule
\vspace{3pt}
Symmetric Balanced Accuracy (SBA) & 
\multicolumn{2}{c}{$\frac{1}{2m}\sum_{i=0}^{m-1} \left( \frac{c_{ii}}{a_{i}} + \frac{c_{ii}}{b_{i}}\right)$} \\
Generalized Means (GM) & $\frac{n\,c_{11} - a_1 b_1}{\sqrt[r]{\frac{1}{2}\left(a_1^r a_0^r + b_1^rb_0^r\right)}}$ & micro / macro / weighted \\
Correlation Distance (CD) & \multicolumn{2}{c}{$\frac{1}{\pi}\arccos(\text{CC})$} \\
\bottomrule
\end{tabular}
\end{small}
\end{table}

Table~\ref{tab:validation_indices} (above the line) lists several widely used classification measures. The most well-known is \emph{accuracy} which is the fraction of correctly classified elements. Accuracy is known to be biased towards the majority class, so it is not appropriate for unbalanced problems. To overcome this, \emph{Balanced Accuracy} re-weights the terms to treat all classes equally. \emph{Cohen's Kappa} uses a different approach to overcome this bias: it corrects the number of correctly classified samples by the expected value obtained by a random classifier~\cite{cohen1960coefficient}. \emph{Matthews Correlation Coefficient} is the Pearson correlation coefficient between true and predicted labelings for binary classification~\cite{gorodkin2004comparing}. For the multiclass case, covariance is computed for each class, and the obtained values are averaged before computing the correlation coefficient. Finally, \emph{Confusion Entropy} computes the entropy of the misclassification distribution for each class and combines the obtained values, see Table~\ref{tab:validation_indices} and~\cite{wei2010CEN} for the details.\footnote{There can be cases when a class is not present in the predicted labels. Then, some measures may contain division by zero. A proper way to fill in such singularities is discussed in Appendix~\ref{sec:more_indices}.}

Some measures are exclusively defined for binary classification. In this case, the classes are often referred to as `positive' and `negative'. \emph{Jaccard}  measures the fraction of correctly detected positive examples among all positive ones (both in true and predicted labelings). \emph{F-measure} is the (possibly weighted) harmonic mean of Recall ($c_{11}/a_1$) and Precision ($c_{11}/b_1$). For measures that do not have a natural multiclass variant, there are several universal extensions obtained via \emph{averaging} the results for $m$ one-vs-all binary classifications~\cite{koyejo2015consistent}. For each one-vs-all classification, a particular class $i$ is considered positive while all other classes are grouped to a negative class.

\emph{Micro averaging} sums up all binary confusion matrices corresponding to $m$ one-vs-all classifications. Formally, it sets true positive as $\sum_{i=0}^{m-1} c_{ii}$, false negative and false positive as $n - \sum_{i=0}^{m-1} c_{ii}$, true negative as $(m-2) n + \sum_{i=0}^{m-1} c_{ii}$. Then, a given binary measure is applied to the obtained matrix.

\emph{Macro averaging} computes the measure values for $m$ binary classification sub-problems and then averages the results: $\frac{1}{m} \sum_{i=0}^{m-1} M(c_{ii}, a_i - c_{ii}, b_i - c_{ii}, n - a_{i} - b_{i} + c_{ii})$, where $M(\cdot)$ is a given binary measure.
Note that macro averaging gives equal weights to all one-vs-all binary classifications. 

In contrast, \emph{weighted averaging} weights one-vs-all binary classifications according to the sizes of the corresponding classes: 
$
\frac{1}{n}\sum_{i=0}^{m-1} a_{i} \cdot M(c_{ii}, a_i - c_{ii}, b_i - c_{ii}, n - a_{i} - b_{i} + c_{ii}).
$

\section{Properties of validation measures}\label{sec:properties}

As clearly seen from the above discussion, there are many options for classification validation. In this section, we propose a formal approach that allows for a better understanding the differences between the measures and for making an informed decision among them for a particular application. For this, we propose properties of validation measures that can be useful across various applications and formally check which measures satisfy which properties. In this regard, we follow the approach proposed in~\cite{gosgens2019systematic} for comparing validation measures for \emph{clustering} tasks.

First, we observe that some theoretical results from~\cite{gosgens2019systematic} are related to~\emph{binary} classification measures. Indeed, a popular subclass of clustering validation measures are~\emph{pair-counting} ones. Such measures are defined in terms of the values $N_{11}, N_{10}, N_{01}, N_{00}$ that essentially define a confusion matrix for binary classification on \emph{element pairs}. Thus, replacing $N_{ij}$ in pair-counting clustering measures by $c_{ij}$, results in \emph{binary} classification measures. We refer to Table~\ref{tab:classification-clusterization-consistency} in Appendix~\ref{sec:more_indices} for the correspondence of some classification and clustering measures. In particular, Accuracy is equivalent to Rand, while Cohen's Kappa corresponds to Adjusted Rand. This equivalence allows us to transfer some of the results from~\cite{gosgens2019systematic} to the context of binary classification. However, an important contribution of our work is the extension of the properties and analysis to the multiclass case. We also prove an impossibility theorem stating that some of the desirable properties cannot be simultaneously satisfied and develop a new family of measures having all properties except one.

Similarly to~\cite{gosgens2019systematic}, we note that all the discussed properties are invariant under linear transformations and interchanging true and predicted labelings. Hence, we may restrict to measures for which higher values indicate higher similarity between classifications.

Table~\ref{tab:properties} summarizes our findings: for each measure, we mathematically prove or disprove each desirable property. Further in this section, we refer only to known measures (above the line), while the remaining ones will be defined and analyzed in Section~\ref{sec:analysis}. In addition to individual measures, we also analyze the properties of micro, macro, and weighted multiclass averagings: for each averaging, we analyze whether it preserves a given property, assuming the binary classification measure satisfies it. All the proofs can be found in Appendix~\ref{sec:checking-properties}. Let us now define and motivate each property.

\begin{table}[t]
\caption{Properties of validation measures and averagings, \cmark/\xmark \, indicates that property is satisfied only in binary case}
\label{tab:properties}
\begin{small}
\begin{center}
\begin{tabular}{l|ccccccccc} 
\toprule
Measure & Max & Min & CSym & Sym & Dist & Mon & SMon & CB & ACB \\
\midrule
$F_1$ (binary) & \cmark & \xmark & \xmark & \cmark & \xmark & \cmark & \xmark & \xmark & \xmark \\
J (binary) & \cmark & \xmark & \xmark & \cmark & \cmark & \cmark & \xmark & \xmark & \xmark \\
CC & \cmark & \cmark/\xmark & \cmark & \cmark & \xmark & \cmark/\xmark & \cmark/\xmark & \cmark & \cmark \\
Acc & \cmark & \cmark & \cmark & \cmark & \cmark & \cmark & \cmark & \xmark & \xmark \\ 
BA & \cmark & \cmark & \cmark & \xmark & \xmark & \cmark & \cmark & \cmark & \cmark \\
$\kappa$ & \cmark & \xmark & \cmark & \cmark & \xmark & \cmark/\xmark & \xmark & \cmark & \cmark \\
CE & \cmark & \xmark & \cmark & \cmark & \xmark & \xmark & \xmark & \xmark & \xmark \\
\midrule
SBA & \cmark & \cmark & \cmark & \cmark & \xmark & \cmark & \cmark & \cmark & \cmark \\
GM (binary) & \cmark & \cmark & \cmark & \cmark & \xmark & \cmark & \cmark & \cmark & \cmark \\
CD & \cmark & \cmark/\xmark & \cmark & \cmark & \cmark & \cmark/\xmark & \cmark/\xmark & \xmark & \cmark \\
\midrule
\multicolumn{10}{c}{Preserving properties by various averaging types} \\
\midrule
Micro & \cmark & \xmark & \cmark & \cmark & \cmark & \cmark & \xmark & \xmark & \xmark \\
Macro & \cmark & \xmark & \cmark & \cmark & \cmark & \cmark & \xmark & \cmark & \cmark \\
Weighted & \cmark & \xmark & \cmark & \xmark & \xmark & \cmark & \xmark & \cmark & \cmark \\
\bottomrule
\end{tabular}
\end{center}
\end{small}
\end{table}

\subsection{Maximal and minimal agreement}
These properties make the upper and lower range of a performance measure interpretable. The \emph{maximal agreement} property requires the measure to have an upper bound that is only achieved when the compared labelings are identical.
\begin{definition}
\label{def:maximal-agreement}
We say that a measure $M$ satisfies \emph{maximal agreement} if there exists a constant $c_{\max}$ such that for all $\mathcal{C}$, $M(\mathcal{C}) \leq c_{\max}$ with equality iff $\mathcal{C}$ is diagonal.
\end{definition}
Also, for a given true labeling, there are several ``worst'' predictions, i.e., labelings that are wrong everywhere. This leads to the following property.
\begin{definition}
\label{def:minimal-agreement}
We say that a measure $M$ satisfies \emph{minimal agreement} if there exists a constant $c_{\min}$ such that for all $\mathcal{C}$, $M(\mathcal{C}) \geq c_{\min}$ with equality iff the diagonal of $\mathcal{C}$ is zero, i.e., $c_{ii} = 0$ for all $i$.
\end{definition}

These properties allow for an easy and intuitive interpretation of the measure's values. While all of the measures in Table~\ref{tab:properties} do satisfy maximal agreement, there are popular measures such as Recall ($c_{11}/a_1$) and Precision ($c_{11}/b_1$) that do not satisfy this property as the maximum can also be achieved when the compared classifications are not identical. For minimal agreement, many performance measures violate it. For example, Cohen's Kappa is obtained from accuracy by subtracting the expected value of accuracy and normalizing the result. As a result of the particular normalization used, it has minimal value ${-}\left(\sum_{i=0}^{m-1}a_ib_i\right)/\left(n^2-\sum_{i=0}^{m-1}a_ib_i\right),$ which is clearly not constant.

If a binary measure satisfies maximal agreement, then its multiclass variant obtained via micro, macro, or weighted averaging also satisfies this property as each one-vs-all binary classification agrees maximally. However, this does not hold for minimal agreement: though each one-vs-all binary classification will have zero true positives, the number of true negatives may still be positive.

\subsection{Symmetry}

\begin{definition}
\label{def:symmetry}
We say that a measure $M$ is \emph{symmetric} if $M(\mathcal{C}) = M(\mathcal{C}^T)$ holds for all $\mathcal{C}$.
\end{definition}
In other words, we require symmetry with respect to interchanging predicted and true labels. This property is often desirable since similarity is usually understood as a symmetric concept. However, in some specific applications, there may be reasons to treat the true and predicted labelings differently and thus use an asymmetric measure. An example of an asymmetric measure is Balanced Accuracy.

Let us also introduce \emph{class-symmetry}, i.e., invariance to permuting the classes.

\begin{definition}
\label{def:class-symmetry}
We say that a measure $M$ is \emph{class-symmetric} if, for any permutation $\pi$ of the classes $\{1,\dots,m\}$ and any confusion matrix $\mathcal{C}$, $M(\mathcal{C})=M(\tilde{\mathcal{C}})$ holds, where $\tilde{\mathcal{C}}$ is given by $\tilde{c}_{ij}=c_{\pi(i),\pi(j)}$. 
\end{definition}

Note that known multiclass measures are all class-symmetric, while in binary classification tasks, there can be an asymmetry between `positive' and `negative' classes. Examples of well-known class-asymmetric binary classification measures are Jaccard and $F_1$.

\subsection{Distance} 
In some applications, it is desirable to have a distance interpretation of a measure: whenever a labeling $A$ is similar to $B$, while $B$ is similar to $C$, it should intuitively hold that $A$ is also somewhat similar to $C$. For instance, it can be the case that the \emph{actual} labels are unknown, and the labeling $A$ is only an approximation of the truth. Then, we would want the similarity between predicted labels and $A$ to be not too different from the similarity between predicted and the actual true labels. This would be guaranteed if the measure is a distance.
\begin{definition}
\label{def:distance}
A measure has \emph{distance} property if it can be linearly transformed to a metric distance.
\end{definition}
A function $d(A,B)$ is a metric distance if it is symmetric, nonnegative, equals zero only when $A=B$, and satisfies the triangle inequality $d(A,C)\leq d(A,B)+d(B,C)$. Note that the first requirement is equivalent to symmetry (Definition~\ref{def:symmetry}), while the second and third imply maximal agreement (Definition~\ref{def:maximal-agreement}). Furthermore, note that if $d$ is a distance, then $c\cdot d$ is also a distance for any $c>0$. Therefore, we can conclude that $M$ is a distance if and only if $M$ satisfies symmetry and maximal agreement while $c_{\max}-M(A,B)$ satisfies the triangle inequality.

While most measures cannot be linearly transformed to a distance, some measures do satisfy this property. For example, the Jaccard measure can be transformed to the Jaccard distance $1-\text{J}(A,B)$. Similarly, Accuracy can be transformed to a distance by $1-\text{Acc}(A,B)$.

\subsection{Monotonicity}

Monotonicity is one of the most important properties of a similarity measure: intuitively, changing one labeling such that it becomes more similar to the other ought to increase the similarity score. Then, to formalize monotonicity, we need to determine what changes make the classifications $A$ and $B$ more similar to each other. The simplest option is to take one element on which $A$ and $B$ disagree and resolve this disagreement.

\begin{definition}
\label{def:monotonicity}
A measure $M$ is \emph{monotone} if $M(\mathcal{C})<M(\tilde{\mathcal{C}})$ for any confusion matrices $\mathcal{C}$ and $\tilde{\mathcal{C}}$ such that $\tilde{\mathcal{C}}$ is obtained from $\mathcal{C}$ by decrementing an off-diagonal entry $c_{ab}$ and incrementing $c_{aa}$ or~$c_{bb}$ and none of the row- or column-sums of $\mathcal{C}$ equal $n$.
\end{definition}

The condition on $\mathcal{C}$ is equivalent to neither $A$ nor $B$ labeling all elements to the same class. We need this to prevent contradictions with the constant baseline property that will be defined in Section~\ref{sec:constant_baseline}. 

Definition~\ref{def:monotonicity} defines a partial ordering over confusion matrices with the same total number of elements. However, we can relax the latter restriction and obtain the following, stronger notion of monotonicity that defines a partial ordering across different numbers of elements.
\begin{definition}
\label{def:strong-monotonicity}
A measure $M$ is \emph{strongly monotone} if $M(\mathcal{C})<M(\tilde{\mathcal{C}})$ for any confusion matrices $\mathcal{C}$ and $\tilde{\mathcal{C}}$ such that 
$\tilde{\mathcal{C}}$ is obtained from $\mathcal{C}$ by either increasing a diagonal entry or decreasing an off-diagonal entry. Here we require that none of the row- or column-sums of $\mathcal{C}$ equal $n$ and that $\mathcal{C}$ and $\tilde{\mathcal{C}}$ are not simultaneously diagonal or zero-diagonal matrices.
\end{definition}
The last condition is needed since otherwise the definition would contradict the maximal (or minimal) agreement properties as $M(\mathcal{C})=c_{\max}\geq M(\tilde{\mathcal{C}})$ holds when $\mathcal{C}$ is diagonal.

All measures in Table~\ref{tab:properties} except for CE and multiclass $\kappa$, CC and CD satisfy monotonicity from Definition~\ref{def:monotonicity}. Strong monotonicity is violated by many measures: for instance, the widely used $F_1$, Jaccard and Cohen's Kappa do not satisfy this intuitive property.

\subsection{Constant baseline}\label{sec:constant_baseline}
The constant baseline is perhaps the most important non-trivial property. On the one hand, it ensures that a measure is not biased towards labelings with particular class sizes $b_1,\dots,b_m$. On the other hand, it also ensures some interpretability for `mediocre' predictions.

Intuitively, if predicted labels are drawn at random and independently of the true labels, we would expect them to have a low similarity with the true labels. Then, if another prediction has a similarly low score, we can say that it is roughly as bad as a random guess. However, this is only possible when such random classifications achieve similar scores, independent of their class sizes. To formalize this, let $U(b_1,\dots,b_m)$ denote the uniform distribution over labelings with class sizes $b_1,\dots,b_m$. We say that the class sizes $b_1,\dots,b_m$ are \emph{unary} if $b_i=n$ for some $i\in\{1,\dots,m\}$. That is, if all elements get classified to the same class, so that $U(b_1,\dots,b_m)$ is a constant distribution.

\begin{definition}
\label{def:constant-baseline}
We say that a measure $M$ has a \emph{constant baseline} property if there exists $c_{\text{base}}(m)$ that does not depend on $n$ but may depend on $m$, such that for any true labeling $A$ and non-unary class sizes $b_1,\dots,b_m$, it holds that $\mathbb{E}_{B\sim U(b_1,\dots,b_m)}[M(A,B)]=c_{\text{base}}(m)$.
\end{definition}

Note that we need to require the class sizes to be non-unary: if these class sizes are unary, we will have contradictions with maximal and minimal agreement when the class sizes of $A$ are also unary. Many popular measures such as $F_1$, Accuracy, and Jaccard do not have a constant baseline. Furthermore, some measures that do have a constant baseline were deliberately designed to have one. For example, Cohen's Kappa was obtained from accuracy by correcting it for chance. While our definition of the constant baseline does allow for a baseline $c_{\text{base}}(m)$ that depends on the number of classes $m$, some measures such as the Matthews Coefficient and Cohen's Kappa have a baseline that is constant w.r.t.~$m$. 

All of the measures that satisfy constant baseline turn out to be linear functions of $c_{ii}$ for fixed class sizes $a_1,\dots,a_m$ and $b_1,\dots,b_m$. For such measures, linearity of the expectation can be utilized to easily compute the baseline  by substituting the expected values $\mathbb{E}_{B\sim U(b_1,\dots,b_m)}[c_{ii}]=\frac{a_ib_i}{n}$. Thus, we also propose the following relaxation of the constant baseline property.

\begin{definition}
\label{def:asymptotic-constant-baseline}
A measure $M$ is said to have an \emph{approximate constant baseline} if there exists a function $c_{\text{base}}(m)$ that does not depend on $n$ but may depend on $m$ such that for any class sizes $a_1,\dots,a_m$ and any non-unary $b_1,\dots,b_m$, $M(\bar{\mathcal{C}})=c_{\text{base}}(m)$, where $\bar{c}_{ij}=\frac{a_ib_j}{n}$.
\end{definition}

The advantage of this relaxation is that it allows us to non-linearly transform measures while still maintaining an approximate constant baseline. Take for example the Matthews Correlation Coefficient: it cannot be linearly transformed to a distance while the transformations $\text{CD}(A,B)=\tfrac{1}{\pi}\arccos(\text{CC}(A,B))$ and $\sqrt{2(1-\text{CC}(A,B))}$ do yield distances. Because Correlation Coefficient has a constant baseline, these non-linear transformations have an approximate constant baseline, see Section~\ref{sec:analysis} for more details.

\medskip

As can be seen from Table~\ref{tab:properties}, there is no measure satisfying all the properties. In particular, there is no measure having both distance and constant baseline. In the next section, we show why this is not a coincidence. 

\section{Impossibility theorem for classification}\label{sec:analysis}

In this section, we focus on binary classification and more deeply analyze the relations between the properties discussed above. Unfortunately, it turns out that the properties introduced in the previous section cannot all be satisfied simultaneously.

\begin{theorem}
\label{thm:impossibility}
There is no binary classification measure that simultaneously satisfies the monotonicity, distance, and constant baseline properties.
\end{theorem}
\begin{proof}
Let $A$ be a labeling with a single positive and $n-1$ negatives. Let $B_1$ be a random labeling with a single positive and let $B_2$ be a random labeling with two positives.
The constant baseline requires $\mathbb{E}[M(A,B_1)]=\mathbb{E}[M(A,B_2)]$, which gives
\[
    \frac{1}{n}c_{\max}+\frac{n-1}{n}M(0,1,1,n-2)=\frac{2}{n}M(1,0,1,n-2)+\frac{n-2}{n}M(0,1,2,n-3),
\]
which we rewrite to
\begin{equation}\label{eq:impossibility_CB}
    2M(1,0,1,n-2)-c_{\max}=(n-1)M(0,1,1,n-2)-(n-2)M(0,1,2,n-3).
\end{equation}
Now, we consider a labeling $C$ with a single positive that does not coincide with the positive of $A$ and a labeling $B$ that has two positives which are the positives of $A$ and $C$.
The triangle inequality tells us that
\[
c_{\max}-M(0,1,1,n-2)\leq 2c_{\max}-M(1,1,0,n-2)-M(1,0,1,n-2)=2(c_{\max}-M(1,1,0,n-2)),
\]
where the last step follows from symmetry (implied by distance).
This is rewritten to
\begin{equation}\label{eq:impossibility_distance}
    2M(1,1,0,n-2)-c_{\max}\leq M(0,1,1,n-2).
\end{equation}
Combining \eqref{eq:impossibility_CB} and \eqref{eq:impossibility_distance}, we obtain
\[
(n-1)M(0,1,1,n-2)-(n-2)M(0,1,2,n-3)\leq M(0,1,1,n-2).
\]
We rewrite this to $M(0,1,1,n-2)\leq M(0,1,2,n-3)$, which clearly contradicts monotonicity.
\end{proof}

Thus, we have to discard one of these properties. Obviously, discarding monotonicity would be highly undesirable since higher values would then not necessarily indicate higher similarity. For this reason, we analyze what happens if we discard either \emph{distance} or~\emph{constant baseline}. All the results stated below are proven in Appendix~\ref{app:analysis}.

\paragraph{Discarding distance} 

Assuming some additional smoothness conditions that are, however, satisfied by all measures discussed in this paper, we prove the following result.

\begin{theorem}
All binary measures that satisfy all properties except distance must be of the form
\[
s\left(\frac{a_0a_1}{n^2},\frac{b_0b_1}{n^2}\right)\cdot\frac{nc_{11}-a_1b_1}{n^2},
\]
where the normalization factor $s(a,b)$ needs to satisfy some additional properties listed in Theorem~\ref{thm:cb-class}.
\end{theorem}

This class of measures is quite wide and contains many unelegant measures. An interesting subclass can be obtained if we normalize by the generalized mean, i.e., take $s(a,b)^{-1}=(\tfrac{1}{2}a^r+\tfrac{1}{2}b^r)^{1/r}$.
\begin{definition}
For $r \in \mathbb{R}$, we define \emph{Generalized Means} measures as
$$\text{GM}_r = \frac{n\,c_{11} - a_1 b_1}{\sqrt[r]{\frac{1}{2}\left(a_1^r a_0^r + b_1^rb_0^r\right)}}\,.$$
\end{definition}

\begin{statement}\label{statement1}
For any $r\in\mathbb{R}$, the measure $\text{GM}_r$ satisfies all properties except for being a distance. 
\end{statement}

We also show that the Generalized Means measures contain two interesting special cases.

\begin{statement}\label{statement2}
If $r\rightarrow0$ (corresponding to the geometric mean), $\text{GM}_r(\C) \to \text{CC}(\C)$. \\
If $r=-1$ (corresponding to the harmonic mean), $\text{GM}_{-1}(\C) = \text{BA}(\mathcal{C})+\text{BA}(\mathcal{C}^\top) - 1$\,.
\end{statement}

Thus, for $r=-1$ Generalized Means is equivalent to the measure $\frac{1}{2}\left(\text{BA}(\mathcal{C})+\text{BA}(\mathcal{C}^\top)\right)$ that we call \emph{Symmetric Balanced Accuracy} (SBA). To the best of our knowledge, this measure has not been used in the classification literature. However, in the clustering literature, a similar measure is known as Sokal\&Sneath's measure~\cite{albatineh2006similarity,gosgens2019systematic}. Interestingly, SBA preserves its properties for the multiclass case.

\begin{statement}
SBA satisfies all properties except for being a distance for any $m \ge 2$.
\end{statement}

\paragraph{Discarding (exact) constant baseline} Note that Theorem~\ref{thm:impossibility} only proves an impossibility for the \emph{exact} constant baseline, but not the \emph{approximate} constant baseline.

\begin{statement}
The measures $\emph{\text{CD}}(A,B):=\tfrac{1}{\pi}\arccos(\emph{\text{CC}}(A,B))$ and  $\emph{\text{CD}}'(A,B):=\sqrt{2(1-\emph{\text{CC}}(A,B))}$ satisfy all properties except the exact constant baseline, but including the approximate constant baseline.
\end{statement}

Following~\cite{gosgens2019systematic}, we call the measure $\tfrac{1}{\pi}\arccos(\text{CC}(A,B))$ \emph{Correlation Distance} (CD). We prove the following result (see Appendix~\ref{app:cb_distance_order} for the details).

\begin{statement}
\emph{CD} approximates a constant baseline with one order of precision better than \emph{CD}$'$.
\end{statement}

Essentially, this is a consequence of the fact that the transformation $\tfrac{1}{\pi}\arccos(\text{CC})$ is a symmetric function around the constant baseline $\text{CC}=0$ while $\sqrt{2(1-\text{CC})}$ is not. In more detail, we show that the leading error term of CD$'$ is of the order $\mathbb{E}[\text{CC}(A,B)^2]$ while the leading error term for CD is of the order $\mathbb{E}[\text{CC}(A,B)^3]$.
Currently, we are not aware of other distance measures for which the constant baseline is approximated up to the same order of precision as CD. We thus argue that for binary classification tasks where a distance interpretation is desirable, Correlation Distance is the most suitable measure.

\section{Inconsistency of measures in practice}\label{sec:experiments}

In this section, we conduct several experiments that demonstrate how often performance measures may disagree in practice in different scenarios. These experiments demonstrate the importance of the problem considered in this paper and show which measures are usually more consistent than others. For binary classification, we consider all measures from Table~\ref{tab:validation_indices}. For F-measure, we take $\beta = 1$, for Generalized Means, we consider $r = 1$. Recall that SBA and CC are also instances of GM with $r = -1$ and $r \to 0$, respectively. Furthermore, Jaccard is a monotone transformation of $F_1$, and CD is a monotone transformation of CC. Therefore, we omit CD and Jaccard from all inconsistency tables. The code for our experiments can be found on GitHub.\footnote{\href{https://github.com/yandex-research/classification-measures}{https://github.com/yandex-research/classification-measures}}

\subsection{Binary measures}\label{sec:exp_binary}

\paragraph{Distinguishing measures for small datasets}

First, we construct simple examples showing the inconsistency of all pairs of binary classification measures. We say that two measures $M_1$ and $M_2$ are consistent on a triplet of classifications $(A, B_1, B_2)$ if $M_1(A,B_1) * M_1(A,B_2)$ implies $M_2(A,B_1) * M_2(A,B_2)$, where $* \in \{>, < , = \}$. Otherwise, we say that the measures are inconsistent. We took $n \in \{2, 3,\ldots, 10\}$ and went through all the possible triplets $(A, B_1, B_2)$ of binary labelings of $n$ elements (we additionally require that all labelings contain both classes). For each triplet, we check which pairs of measures are inconsistent. We say that a pair of measures is indistinguishable for a given $n$ if it is consistent on all triplets.

\begin{wraptable}[12]{r}{0.42\textwidth}
    \vspace{-13pt}
    \captionof{table}{Indistinguishable measures}
    \label{tab:binary_inconsistency}
    \centering
    \begin{small}
    \begin{tabular}{c p{4cm}}
    \toprule
    $n$ & measures \\
    \midrule
    \vspace{2pt}
    2 & [Acc, BA, $F_1$, $\kappa$, CE, GM$_1$, CC, SBA] \\ \vspace{2pt}
    3 & [Acc, BA, $\kappa$, GM$_1$, CC, SBA] \\ \vspace{2pt}
    4-5 & [BA, $\kappa$, GM$_1$, CC, SBA] \\ \vspace{2pt}
    6-7 & [GM$_1$, CC, SBA] \\ \vspace{2pt}
    8 & [CC, SBA] \\ \vspace{2pt}
    9-10 & --- \\
    \bottomrule
    \end{tabular}
    \end{small}
\end{wraptable}
Table~\ref{tab:binary_inconsistency} lists all measures that are indistinguishable for a given $n$. For instance, for $n=2$, all measures are always consistent. For $n=4$, we can distinguish Acc, $F_1$, and CE from other measures and each other. Interestingly, the remaining measures are those having the constant baseline property. Importantly, the most consistent measures are CC, SBA, and GM$_1$~--- these measures have the best properties according to our analysis. This supports our intuition that ``good'' measures agree with each other better than those having fewer desired properties. Additionally, in Appendix~\ref{app:binary}, we list six triplets $(A, B_1, B_2)$ with $n = 10$ that discriminate all pairs of different measures.

\paragraph{Experiment within a weather forecasting service}

In this experiment, we aim at understanding whether the differences between measures may affect the decisions made while designing real systems. For this purpose, we conduct an experiment within the \emph{Yandex.Weather} service.

There is a model that predicts the presence/absence of precipitation at a particular location~\cite{lebedev2019precipitation}. The prediction is made for 12 prediction intervals (\emph{horizons}): from ten minutes to two hours. The original model returns the probability of precipitation, which can be converted to binary labels via a threshold. There are six thresholds used in this experiment, which lead to six different models. The measures were logged for 12 days. To sum up, for each threshold (model), each day, and each horizon, we have a confusion matrix that can be used to compute a performance measure.

\begin{wraptable}[13]{r}{0.65\textwidth}
    \vspace{-4pt}
    \caption{Inconsistency of binary measures for rain prediction, \%}
    \label{tab:weather}
    \vspace{4pt}
    \centering
    \begin{small}
    \begin{tabular}{l|rrrrrrrr}
    \toprule
 & Acc & BA\, & $F_1$\, & $\kappa$\,\,\, & CE\, & GM$_1$ & CC\, & SBA \\
\midrule
Acc & --- & 96.5 & 41.0 & 37.5 & 3.1 & 38.7 & 44.3 & 55.9 \\
BA & 96.5 & --- & 55.6 & 58.9 & 99.7 & 57.7 & 52.0 & 40.4 \\
$F_1$ & 41.0 & 55.6 & --- & 3.3 & 44.2 & 2.2 & 3.4 & 15.0 \\
$\kappa$ & 37.5 & 58.9 & 3.3 & --- & 40.7 & 1.1 & 6.7 & 18.3 \\
CE & 3.1 & 99.7 & 44.2 & 40.7 & --- & 41.9 & 47.5 & 59.1 \\
GM$_1$ & 38.7 & 57.7 & 2.2 & 1.1 & 41.9 & --- & 5.5 & 17.1 \\
CC & 44.3 & 52.0 & 3.4 & 6.7 & 47.5 & 5.5 & --- & 11.4 \\
SBA & 55.9 & 40.4 & 15.0 & 18.3 & 59.1 & 17.1 & 11.4 & --- \\
\bottomrule
\end{tabular}
\end{small}
\end{wraptable}
For each pair of measures, we compute how often they are inconsistent according to the definition above. For this, we aggregate the results over all days and horizons. Table~\ref{tab:weather} shows that there are pairs of measures with substantial disagreement: e.g., accuracy and Balanced Accuracy almost always disagree. This can be explained by the fact that accuracy has a bias towards the majority class, so it prefers a higher threshold, while Balanced Accuracy weighs true positives more heavily, so it prefers a lower threshold. In contrast, GM$_1$, CC, $\kappa$, and $F_1$ agree with each other much better. In Appendix~\ref{app:binary} we conduct a more detailed analysis. In particular, we separately consider the ten-minute and two-hour prediction horizons and show that the behavior and consistency of measures significantly depend on the horizon as the horizon defines the balance between true positives, true negatives, false positives, and false negatives. We also observe that CC and SBA perfectly agree for the ten-minute horizon but have noticeable disagreement for two hours.

\subsection{Multiclass measures}\label{sec:multiclass}

In this section, we analyze multiclass measures. For all measures that are defined for the multiclass problems, we consider their standard expressions (if not stated otherwise). For other measures ($F_1$, Jaccard, GM$_1$), we use macro averaging.

\paragraph{Image classification}

We conduct an experiment on ImageNet~\cite{russakovsky2015imagenet}, a classic dataset for image classification. For this, we take the top-10 algorithms that are considered to be state-of-the-art at the moment of submission.\footnote{\href{https://github.com/rwightman/pytorch-image-models/blob/master/results/results-imagenet.csv}{https://github.com/rwightman/pytorch-image-models/blob/master/results/results-imagenet.csv} \\ (May 8, 2021).} 
We check whether the leaderboard based on accuracy is consistent with the leaderboards based on other measures. Thus, we apply the models to the test set, compute the confusion matrices, and compare all measures defined in Table~\ref{tab:validation_indices}.

Notably, the ImageNet dataset is balanced. This makes all measures more similar to each other. For instance, accuracy and BA are equal in this scenario. Also, the \emph{constant baseline} property discussed in Section~\ref{sec:constant_baseline} is especially important for \emph{unbalanced} datasets. Thus, measures are \emph{more consistent} on balanced data. Nevertheless, we notice that the ranking can be inconsistent starting from the algorithm ranked fifth on the leaderboard.

The (partial) results are shown in Table~\ref{tab:imagenet}. Here we compare EfficientNet-B7 NoisyStudent~\cite{xie2020self} and Swin-B Transformer (patch size 4x4, window size 12x12, image size $384^2$)~\cite{liu2021swin} that are the fifth and sixth models in the leaderboard. One can see that the measures inconsistently rank the algorithms: Confusion Entropy, Jaccard, and SBA disagree with accuracy and other measures. Interestingly, while Jaccard and $F_1$ always agree for binary problems, they may disagree after the macro averaging, as we see in this case. Also, for one measure, different multiclass extensions may be inconsistent, as we see with macro averaging versus the standard definition of the multiclass Correlation Coefficient. More detailed results can be found in Appendix~\ref{app:multiclass}.

\begin{table}
\caption{Inconsistent results on ImageNet, \% (fifth and sixth models in the leaderboard)}
\label{tab:imagenet}
\vspace{4pt}
\centering
\begin{small}
\begin{tabular}{l|ccccccccc}
\toprule
 & Acc/BA & $F_1$ & J & $\kappa$ & $1-$CE & GM$_1$ & CC & CC$^{macro}$ & SBA \\
\midrule
Efficientnet  & \textbf{86.46} & \textbf{86.30} & 77.525 & \textbf{86.44} & 93.41 & \textbf{86.28} & \textbf{86.44} & 86.419 & 86.57 \\
Swin & 86.43 & 86.27 & \textbf{77.531} & 86.42 & \textbf{93.51} & 86.26 & 86.42 & \textbf{86.423} & \textbf{86.61} \\
\bottomrule
\end{tabular}
\end{small}
\end{table}

\paragraph{Sentiment analysis} In the previous experiment, we noticed that despite several disagreements, the measures usually rank the algorithms similarly. This can be caused by the fact that the test set of ImageNet is balanced: all classes have equal sizes. However, in practical applications, we rarely encounter balanced data. Thus, we also consider an unbalanced classification task. In this experiment, we take the 5-class Stanford Sentiment Treebank (SST-5) dataset~\cite{socher2013recursive}. We compare the following algorithms: TextBlob, VADER, Logistic Regression, SVM, FastText, Flair+ELMo, and Flair+BERT~\cite{SST}. Table~\ref{tab:sst} shows that different measures rank the algorithms differently. Among the measures shown in the table, the only consistent rankings are the one provided by $\kappa$ and BA and the second given by $F_1$, GM, and Jaccard. Note that the latter ranking significantly disagrees with the ranking by accuracy.

\begin{table}
\caption{Ranking algorithms according to different measures on SST-5: from 1 (best) to 7 (worst)}
\label{tab:sst}
\centering
\vspace{4pt}
\begin{small}
\begin{tabular}{l|cccccccccc}
\toprule
 & Acc & BA & $F_1$ & J & $\kappa$ & CE & GM$_1$ & CC & CC$^{macro}$ & SBA \\
\midrule
Flair+ELMo & 1 & 1 & 1 & 1 & 1 & 1 & 1 & 1 & 1 & 1 \\
Flair+BERT & 2 & 4 & 5 & 5 & 4 & 2 & 5 & 2 & 2 & 2 \\
SVM & 3 & 3 & 3 & 3 & 3 & 5 & 3 & 3 & 4 & 4 \\
Logistic & 4 & 5 & 4 & 4 & 5 & 3 & 4 & 5 & 5 & 3 \\
FastText & 5 & 2 & 2 & 2 & 2 & 6 & 2 & 4 & 3 & 5 \\
VADER & 6 & 6 & 6 & 6 & 6 & 7 & 6 & 6 & 6 & 7 \\
TextBlob & 7 & 7 & 7 & 7 & 7 & 4 & 7 & 7 & 7 & 6 \\
\bottomrule
\end{tabular}
\end{small}
\end{table}

\vspace{7pt}

Appendix~\ref{app:multiclass} contains an additional experiment with an unbalanced multiclass dataset, where we show the inconsistency rates of the considered measures and different multiclass extensions.

\section{Conclusion and future work}\label{sec:conclusion}

In this paper, we propose a systematic approach to the analysis of classification performance measures: we propose several desirable properties and theoretically check each property for a list of measures. We also prove an impossibility theorem: some desirable properties cannot be simultaneously satisfied, so either distance or \emph{exact} constant baseline has to be discarded.

Based on the properties we analyzed in this paper, we come to the following practical suggestions. If the distance requirement is needed, Correlation Distance seems to be the best option: it satisfies all the properties except for the exact constant baseline, which is still approximately satisfied. Otherwise, we suggest using one of Generalized Means, including Correlation Coefficient and Symmetric Balanced Accuracy~--- they satisfy all the properties except distance. For binary classification, CC is a natural choice as it can be non-linearly transformed to a distance. For multiclass problems, Symmetric Balanced Accuracy has an additional advantage: among the considered measures, only this one preserves its good properties in the multiclass case. Finally, we do not advise using averagings, but if needed, then macro averaging preserves more properties.

There are still many open questions and promising directions for future research. First, we would like to see whether one could construct a set of desirable properties that can be used as axioms to uniquely define one good measure (or a parametrized group of measures). Secondly, it is an open problem whether Generalized Means measures in general (or SBA in particular) can be converted to a distance via a continuous transformation. Finally, our work does not cover ranking and probability-based measures. Thus, we leave aside such widely used measures as cross-entropy and AUC. Formalizing and analyzing their properties is an important direction for future research.

\paragraph{Broader impact} Our work may help towards reducing certain biases in research. For instance, some measures (e.g., accuracy) are biased towards the majority class. Thus, the bias towards the majority class could be even amplified with the poor metric selection. Our work could provide some clues on how to avoid such a situation.

\newpage
\begin{ack}
Part of this work was done while Martijn G\"{o}sgens was visiting Yandex and Moscow Institute of Physics and Technology (Russia). The work of Martijn G\"{o}sgens is supported by the Netherlands Organisation for Scientific Research (NWO) through the Gravitation NETWORKS grant no. 024.002.003.

The authors would like to thank Alexander Ganshin, Pert Vytovtov, and Eugenia Elistratova for providing the weather forecasting data.
\end{ack}

\bibliographystyle{abbrv}
\bibliography{main}


\newpage

\newpage

\appendix

\section{Related work}\label{sec:related}

Properly choosing an evaluation measure is a significant problem that attracted much attention in recent and long-standing research. In this section, we cover some related papers. In summary, while there are many related studies, the field lacks systematic approaches. Some papers focus on particular advantages and flaws of particular measures, while others suggest some informal properties. Our paper suggests a unified analysis that generalizes and extends the existing research.

A work conceptually related to ours is~\cite{sebastiani2015axiomatically}. In this paper, the authors define a list of properties (they refer to them as \emph{axioms}). Some properties are similar to ours: MON is our monotonicity, FIX is somewhat similar (but not the same) to our maximal and minimal agreement, CHA is the constant baseline, and SYM is our class-symmetry. The properties CON and SDE/WDE are related to singularities. In the current paper, we do not focus on singularities since measures are naturally extended to such cases, as we discuss in Section~\ref{sec:more_indices}. Another property is called Robustness to Imbalance (IMB). This property requires a constant classifier that classifies all elements to either the positive or the negative class to get a constant similarity score $k_1$ or $k_2$, respectively. One can see that our constant baseline thus implies IMB with $k_1=k_2$. On the other hand, having $k_1 \neq k_2$ may lead to bias towards a particular class, which does not seem to be desired. The authors show that several known measures do not satisfy some of the properties and propose \emph{K measure}, which is a shifted version of Balanced Accuracy with singularities properly resolved. Let us also note that the authors advocate against CC largely because they do not use this same straightforward resolution to the singularities for this measure. Our work differs in the following aspects. First, we consider more comprehensive lists of measures and properties and check each property for each popular measure. In particular, our properties include symmetry (in terms of interchanging labelings), distance, and approximate constant baseline. We show that in terms of the extended list of properties, there are better variants than the K measure (which we refer to as Balanced Accuracy). We also provide a deep theoretical analysis of properties and propose a new family of `good' measures. In addition, we rigorously analyze the multiclass scenario, including the properties of aggregation schemes. To sum up, while there are methodological similarities, there are significant differences in the analysis and outcomes.

With some similarities to our research, the authors of~\cite{hossin2015review} formulate a list of (informal) properties that are argued to be desirable for an evaluation measure. These properties include having a natural extension to the multiclass case, low complexity and computational cost, distinctiveness and discriminability, informativeness, and favoring the minority class. While informativeness seems to be an informal analog of our \emph{constant baseline}, the properties are not formally defined, and thus no systematic analysis of measures with respect to the properties can be given.

Another work related to our research~\cite{sokolova2009systematic} defines a list of properties by describing several transformations of the confusion matrix that should not change the measure value. As a result, the authors provide a table listing which measures are invariant under which transformations. This analysis includes our \emph{symmetry} and also \emph{scale invariance} which we discuss further in Appendix~\ref{app:analysis}. However, the discussed properties are quite simple, and the work does not cover the most important and complex ones like \emph{constant baseline}, \emph{monotonicity}, or \emph{distance}.

There are papers focusing on properties of a particular measure, for instance, Cohen's Kappa~\cite{delgado2019cohen,powers2012problem}, Confusion Entropy~\cite{delgado2019enhancing}, or Balanced Accuracy~\cite{brodersen2010balanced}. Some papers go beyond the threshold measures considered in our paper. For instance,~\cite{cortes2004auc} theoretically analyzes how the area under the ROC curve (AUC) relates to accuracy. Another work focusing on AUC and accuracy is~\cite{huang2005using}. This paper formally defines two properties: \emph{degree of consistency} and \emph{degree of discriminancy}. The degree of consistency is not a property of a measure but rather a property of a \emph{pair} of measures. In our experiments on synthetic and real data, we compute such degrees of (in)consistency. The degree of discriminancy, in turn, can be reformulated as the number of different values that a measure has (in a given domain).

There are studies advocating using the Matthews correlation instead of some other popular measures. For instance, the authors of~\cite{delgado2019cohen} compare CC to Cohen's Kappa and show that the latter may have undesirable behavior in some scenarios. Essentially, these scenarios show that Cohen's Kappa does not satisfy our strong monotonicity requirement. A recent paper~\cite{chicco2020advantages} advocates using CC over $F_1$ and accuracy based on several intuitive \emph{use cases}, where it is clear that the performance is poor, but only CC can correctly detect that in all cases. We note that all the use cases are related to our \emph{constant baseline} property. Similarly to the above research, we conclude that CC should be preferred over $F_1$, accuracy, and Cohen's kappa. Importantly, our conclusion is based on a rigorous analysis and formal properties. 


Numerous studies empirically compare different classification measures~\cite{choi2010survey,ferri2009experimental}; some of them specifically focus on imbalanced data~\cite{luque2019impact}. Going beyond particular measures, some studies compare the properties of micro- and macro- averagings~\cite{van2013macro}. However, to the best of our knowledge, our work is the first one giving a formal approach to the problem. 

Finally, as we discuss in the main text in more detail, our work is motivated by a recent study~\cite{gosgens2019systematic} that analyzes properties of \emph{cluster validation} measures. We refer to this paper for an overview of related work in cluster analysis.

\begin{table}
    \caption{Notation}
    \label{tab:notation}
    \vspace{4pt}
    \centering
    \begin{tabular}{cp{240pt}}
    \toprule
    Variable & Definition \\
    \midrule
    \vspace{3pt}
     n  & number of elements \\
     \vspace{3pt}
     m & number of classes \\
     \vspace{3pt}
     $c_{ij}$ & number of elements of class $i$ that are predicted as $j$ \\
     \vspace{3pt}
     $A_i$ & elements with true label $i$ \\
     \vspace{3pt}
     $B_i$ & elements with predicted label $i$ \\
     \vspace{3pt}
     $\mathcal{C} = (c_{ij})$ & $m \times m$ confusion matrix \\
     \vspace{3pt}
     $a_{i} = \sum_{j=0}^{m-1} c_{ij}$ & size of $i$-th class in the true labeling \\
     \vspace{3pt}
     $b_i = \sum_{j=0}^{m-1} c_{ji}$ & size of $i$-th class in the predicted labeling \\
     \vspace{3pt}
     $p_A = \frac{a_1}{n}, \,p_B=\frac{b_1}{n}$ & fraction of positive entries (for binary classification) \\
     \vspace{3pt}
     $p_{AB} = \frac{c_{11}}{n}$ & fraction of agreeing positives (for binary classification) \\
     \vspace{3pt}
     $M(\mathcal{C})$, $M(A,B)$, $M(p_{AB},p_A,p_B)$ & classification validation measure \\
    \bottomrule
    \end{tabular}
\end{table}

\section{More on classification validation measures}\label{sec:more_indices}

\paragraph{Notation}
For convenience, Table~\ref{tab:notation} lists notation frequently used throughout the text.

\paragraph{Resolving singularities}

When some of the classes are not present in the predicted (or, more rarely, true) labelings, some measures from Table~\ref{tab:validation_indices} may not be defined. Let us discuss how to resolve such singularities appropriately.

For some measures, singularities can only occur when the measures maximally or minimally agree with each other. For example, the denominator of Jaccard is only zero if $a_1=b_1=0$, in which case $A=B$ must hold so that the singularity is easily resolved by maximal agreement, leading to $J(A,B)=1$.

For measures such as Matthews Coefficient, singularities can be resolved using constant baseline. For CC, a singularity can only occur whenever either $n^2=\sum_{m=1}^na_i^2$ or $n^2=\sum_{m=1}^nb_i^2$. This implies that either $A$ or $B$ classifies all elements to the same class. If both $A$ and $B$ classify all elements to the same class, then the singularity can be resolved by maximal agreement (if they classify to the same class) or minimal agreement (otherwise). If one of $A$ and $B$ classifies all elements to the same class, then the constant baseline tells us that $M(A,B)=0$ should hold.

Similarly, some measures, e.g., BA and SBA, contain terms $c_{ii}/a_i$ (or $c_{ii}/b_i$) that may have singularities. In cases where $a_i=0$, these singularities can be algebraically resolved by $c_{ii}=0=\frac{a_ib_i}{n}$. This leads to $\frac{c_{ii}}{a_i}=\frac{b_i}{n}$ and ensures that such singularities will not lead to violations of constant baseline.

\paragraph{Correspondence with pair-counting cluster validation measures} 
As discussed in the main text, there is a correspondence between pair-counting cluster validation measures and binary classification validation measures. We refer to Table~\ref{tab:classification-clusterization-consistency} for some corresponding pairs.

\begin{table}[t]
\caption{Correspondence of binary classification measures and pair-counting clustering measures}
\begin{center}
\begin{tabular}{ll} 
\toprule
Classification & Clustering \\
\midrule
$F_1$ & Dice\\
Jaccard & Jaccard \\
Matthews Correlation Coefficient & Pearson Correlation Coefficient \\
Accuracy & Rand \\
Cohen's Kappa & Adjusted Rand \\
\midrule
Symmetric Balanced Accuracy & Sokal\&Sneath \\
Correlation Distance & Correlation Distance \\
\bottomrule
\end{tabular}
\end{center}
\label{tab:classification-clusterization-consistency}
\end{table}

\section{Checking the properties}
\label{sec:checking-properties}

Table~\ref{tab:properties} lists which measures satisfy the discussed properties and which averaging schemes preserve them. In this section, we formally prove all the results. Recall that if a measure does not have a natural extension to the multiclass case, then we analyze its binary variant. Additionally, if a property is violated in the binary case, then we do not check it in the multiclass case.

\paragraph{Using existing analysis of cluster validation indices}

As discussed in the previous section, there is a correspondence between some pair-counting clustering evaluation measures and classification ones. Recall that a pair-counting clustering measure is a function of $N_{11}$, $N_{10}$, $N_{01}$, and $N_{00}$, where $N_{11}$ is the number of element-pairs belonging to the same cluster in both partitions, $N_{00}$ is the number of pairs belonging to different clusters in both partitions, $N_{10}$ is the number of pairs belonging to the same cluster in the true partition but to different clusters in the predicted partition, and $N_{01}$ is the number of pairs belonging to different clusters in the true partition but to the same cluster in the predicted partition. Thus, pair-counting clustering measures are functions of TP, TN, FP, and FN defined for \emph{classifying element-pairs} into ``intra-cluster'' and ``inter-cluster'' pairs. So, replacing $N_{ij}$ by $c_{ij}$ we naturally get a binary classification measure. Some classification evaluation indices have been theoretically analyzed in~\cite{gosgens2019systematic}. Using Table~\ref{tab:classification-clusterization-consistency}, we can adopt some of these results for classification measures. 

\subsection{Maximal and minimal agreement}\label{sec:max-min-agreement}

To check whether a measure has the maximal or minimal agreement properties, we substitute the entries of a diagonal matrix or a matrix with zero diagonal into the expression: we need either a strict upper or a strict lower bound for the measure values. Note that for measures having the monotonicity property (i.e., for all considered measures except CE and multiclass $\kappa$, CC, CD), it is sufficient to check that we obtain constant values for diagonal and non-diagonal matrices. Indeed, each confusion matrix can be monotonically transformed to a diagonal (or a zero-diagonal) one.

By substituting a diagonal confusion matrix, we get the maximal agreement for $F_1$, J, CC, Acc, BA, $\kappa$, SBA, and GM$_r$ with $c_{\max} = 1$. For $-$CD, the maximal agreement holds with $c_{\max} = 0$. Finally, $\text{CE} = 0$ if $\mathcal{C}$ is diagonal and otherwise there exists a pair $(i, j)$ such that $c_{ij} > 0, a_i > 0, b_j > 0$, so we get $-\text{CE} < 0$.

The minimal agreement for accuracy, Balanced Accuracy, and Symmetric Balanced Accuracy clearly holds with $c_{\min} = 0$. Substituting a zero-diagonal confusion matrix into GM$_r$, we get $c_{\min} = -1$.

For binary measures $F_1$ and Jaccard, the minimal agreement does not hold: these measures equal zero not only for zero-diagonal matrices, but also when $c_{11} = 0$ and $c_{00} > 0$.

In the binary case, the minimal agreement of CC is satisfied with $c_{\min} = -1$. However, this property is violated if $m > 2$. For instance, consider the confusion matrices
$
\mathcal{C}_1 = 
\left(\begin{smallmatrix}
    0 & 1 & 0\\
    0 & 0 & 1\\
    2 & 0 & 0
\end{smallmatrix}\right)
$ and
$
\mathcal{C}_2 = 
\left(\begin{smallmatrix}
    0 & 1 & 0\\
    1 & 0 & 1\\
    0 & 1 & 0
\end{smallmatrix}\right)
$.
We have $\text{CC}(\mathcal{C}_1) \neq \text{CC}(\mathcal{C}_2)$ (-0.5 and -0.6, respectively), while $\mathcal{C}_1$ and  $\mathcal{C}_2$ are both zero-diagonal. Note that CD is a monotone transformation of CC, so CD inherits the same properties.

For CE, the minimal agreement does not hold even in the binary case~\cite{delgado2019enhancing}: let $\mathcal{C}_1 = 
\left(\begin{smallmatrix}
    0 & 6\\
    6 & 0\\
\end{smallmatrix}\right)
$
and 
$\mathcal{C}_2 = 
\left(\begin{smallmatrix}
    1 & 5\\
    5 & 1\\
\end{smallmatrix}\right)\,.
$
Then, we have $\text{CE}(C_1) = 1$ and $\text{CE}(C_2) > 1$. This contradicts both the minimal agreement and monotonicity properties.

Finally, substituting a zero-diagonal matrix into Cohen's Kappa, we get $\frac{-\sum_i a_i b_i}{n^2 - \sum_i a_i b_i}$ which is clearly non-constant.

\subsection{Symmetry}

\paragraph{Class-symmetry}
Almost all considered measures are class-symmetric: they do not change after interchanging class labels. The only exceptions are $F_1$ and Jaccard. Class-symmetry of GM follows from the fact that it can be rewritten as $\left(c_{11} c_{00} - c_{01} c_{10}\right) / \left(\sqrt[r]{\frac{1}{2}\left(a_1^r a_0^r + b_1^r b_0^r\right)}\right)$.

\paragraph{Symmetry}
This property is easily verified by swapping $a_i$ with $b_i$ and $c_{ij}$ with $c_{ji}$. Thus, all measures except BA are symmetric.

\subsection{Distance}

We refer to~\cite{kosub2019note} for the proof that Jaccard satisfies this requirement. To show that accuracy has this property, we need to show that $1 - \text{Acc}$ is a distance, which is true since $n(1 - \text{Acc})$ is the Hamming distance.

Now, we need to prove that CD is a distance since it was previously known only for the binary case.
\begin{lemma}
The Correlation Distance $\text{\emph{CD}}=\tfrac{1}{\pi}\arccos(\text{\emph{CC}})$ is a distance for any $m\geq2$.
\end{lemma}
\begin{proof}
Let us represent a classification by a matrix via one-hot encoding, i.e., $A=(a_{ij})_{i\in[n],j\in[m]}$, where $a_{ij}=\1\{A(i)=j\}$, and define $a_j=\sum_ia_{ij}$. Note that for two labelings $A$ and $B$, the Frobenius inner product is given by
\[
\langle A,B\rangle =\sum_j c_{jj},
\]
where $c_{jj}$ is the $j$-th diagonal entry of the confusion matrix for $A$ and $B$. 
Next, we define 
\[
\bar A := \left(a_{ij}-\tfrac{a_j}{n}\right)_{i\in[n],j\in[m]}.
\]
Then, for two labelings $A$ and $B$, the Frobenius inner product of these mappings is given by
\[
\langle\bar A,\bar B\rangle=\sum_j \left( c_{jj}-\frac{a_jb_j}{n} \right).
\]
And the squared length equals
\[
\|\bar A\|^2=n-\frac{\sum_ja_j^2}{n}.
\]
Therefore, we get
\[
\text{CC}(\mathcal{C})=\frac{\langle \bar A,\bar B\rangle}{\|\bar A\|\cdot\|\bar B\|},
\]
so that its arccosine is indeed the angle between $\bar A$ and $\bar B$, which is a metric distance.
\end{proof}

Let us now prove that the remaining measures cannot be linearly transformed to metric distances. According to Theorem~\ref{thm:impossibility}, a measure that satisfies monotonicity and constant baseline cannot have the distance property. This proves that CC, BA, $\kappa$, SBA, and GM$_r$ cannot be linearly transformed to a distance (note that BA is also not symmetric). To show that CE does not have this property, we take $A = (1, 1, 0)$, $B = (1, 1, 1)$, $C = (1, 0, 1)$. Note that $\text{CE}\left(A, C\right) = 1$ and $\text{CE}\left(A, B\right) = \text{CE}\left(B, C\right) \approx 0.387$. Hence, $\text{CE}\left(A, C\right) > \text{CE}\left(A, B\right) + \text{CE}\left(B, C\right)$ that disproves the distance property. Finally, the counter-example for $F_1$ is given in~\cite{gosgens2019systematic} since $F_1$ is equivalent to the Dice index.

\subsection{Monotonicity}
\label{subsec:monotonicity}

\paragraph{Strong monotonicity}

$F_1$ and Jaccard are constant w.r.t.~$c_{00}$, so they are not strongly monotone. Cohen's Kappa also violates this property~\cite{gosgens2019systematic}: we have 
$\kappa\left(\begin{smallmatrix}
    1 & 2\\
    1 & 0\\
\end{smallmatrix}\right) < 
\kappa\left(\begin{smallmatrix}
    1 & 3\\
    1 & 0\\
\end{smallmatrix}\right)$.
Then, CE is not strongly monotone since it is not monotone (see below).

The fact that CC is strongly monotone in the binary case is proven in~\cite{gosgens2019systematic} (for general binary vectors). In contrast to the binary case, CC is not strongly monotone if $m \ge 3$ since it is not monotone. CD inherits monotonicity properties from CC.

To prove that accuracy is strongly monotone, we use the inequality $(a + x) / (b + x) > a / b$ for $b>a>0$ and $x>0$. So, accuracy increases if we simultaneously increment $c_{ii}$ (for some $i$) and $n$. If we increment $n$ and $c_{ij}$ for $i \neq j$, then accuracy decreases, which proves strong monotonicity. Similar reasoning works for BA and SBA. 

Finally, let us prove that GM$_r$ is strongly monotone for any~$r$.

\begin{lemma}
\label{lem:gm-strong-monotonicity}
\emph{GM}$_r$ is strongly monotone.
\end{lemma}

\begin{proof}
Note that $r\to 0$ corresponds to CC. Since this measure is considered above, we may assume that $r \neq 0$.

Due to the symmetry of GM, we only need to prove that the measure is strongly monotone w.r.t.~$c_{11}$ and $c_{10}$. Moreover, GM flips the sign if we invert the labels in one classification. Hence, we only need to prove that it is increasing in $c_{11}$. Considering GM as a function of independent variables $c_{11}, c_{00}, c_{01}, c_{10}$, we calculate
\begin{align*}
     \frac{\partial\text{GM}_r}{\partial c_{11}} &= \left(n + c_{11} - b_1 - a_1\right)\left(\frac{1}{2}\left(a_1^r a_0^r + b_1^r b_0^r\right)\right)^{-1/r} \\
     &-\frac{1}{2r}\left(a_1^{r-1} a_0^r r + b_1^{r - 1} b_0^r r\right)\left(n c_{11} - a_1 b_1\right)\left(\frac{1}{2}\left(a_1^r a_0^r + b_1^r b_0^r\right)\right)^{-1/r - 1}.
\end{align*}
Simplifying the expression, we note that it has the same sign as the following sum
\begin{multline*}
    \left(n + c_{11} - b_1 - a_1\right)\left(a_1^r a_0^r + b_1^r b_0^r\right) - \left(a_1^{r-1} a_0^r + b_1^{r - 1} b_0^r\right)\left(n c_{11} - a_1 b_1\right) \\
    = a_0^r a_1^{r-1}\left(-n c_{11} + a_1 b_1 + a_1 n + a_1 c_{11} - b_1 a_1 - a_1^2\right) \\
    +b_0^r b_1^{r-1}\left(-n c_{11} + a_1 b_1 + b_1 n + b_1 c_{11} - b_1 a_1 - b_1^2\right) \\
    =a_0^r a_1^{r-1}\cdot a_0c_{10} + b_0^r b_1^{r-1}\cdot b_0c_{01} \geq 0.
\end{multline*}
Note that the last expression is strictly positive if the classifications $A$ and $B$ do not coincide and are not constant.
\end{proof}

\paragraph{Monotonicity}

First, we note that monotonicity of Acc, BA, SBA, and GM follows from their strong monotonicity. Monotonicity of $F_1$ and Jaccard follows from their definitions, see also~\cite{gosgens2019systematic}.

Monotonicity of CC in the binary case follows from its strong monotonicity. However, for $m \ge 3$, CC is not monotone. Indeed, consider
$\mathcal{C}_1 = 
\left(\begin{smallmatrix}
    1 & 0 & 0\\
    6 & 1 & 0\\
    0 & 0 & 1\\
\end{smallmatrix}\right)$,
$\mathcal{C}_2 = 
\left(\begin{smallmatrix}
    1 & 0 & 0\\
    7 & 0 & 0\\
    0 & 0 & 1\\
\end{smallmatrix}\right)$
and note that $\text{CC}(C_2) > \text{CC}(C_1)$.

The fact that Cohen's Kappa is monotone follows from~\cite{gosgens2019systematic} (the proof for Adjusted Rand applies to general binary vectors). Similarly to CC, for $m \ge 3$, monotonicity is violated. Consider, for example,
$\mathcal{C}_1 = 
\left(\begin{smallmatrix}
    0 & 1 & 2\\
    0 & 0 & 0\\
    1 & 0 & 0\\
\end{smallmatrix}\right)$,
$\mathcal{C}_2 = 
\left(\begin{smallmatrix}
    1 & 0 & 2\\
    0 & 0 & 0\\
    1 & 0 & 0\\
\end{smallmatrix}\right)$ and note that $\kappa(C_1) > \kappa(C_2)$.

Finally, the example from Section~\ref{sec:max-min-agreement} disproves the monotonicity of CE.

\subsection{Constant baseline}

\paragraph{Approximate constant baseline}

Substituting $c_{ij} = a_i b_j / n$ into CC, CD, BA, $\kappa$, SBA, and GM, we get values that do not depend on $a_i, b_i$. Thus, these measures have the approximate constant baseline property.

Substituting $c_{ij} = a_i b_j / n$ into $\text{CE}$, we get that the result depends on $a_i$ and $b_j$. For instance, taking $(a_0, a_1) = (2, 1), (b_0, b_1) = (1, 2)$ and  $(a_0, a_1) = (0, 3), (b_0, b_1) = (1, 2)$ we get different values of $\text{CE}$ that disproves approximate constant baseline. Similarly, $F_1$, Jaccard, and accuracy do not have this property.

\paragraph{Exact constant baseline}

We will use the following lemma.
\begin{lemma}\label{lem:constant-baseline-expectation}
Suppose that the fixed true labeling $A$ has class-sizes $a_1, \ldots, a_m$, while the predicted labeling $B \sim U(b_1, \ldots, b_m)$ is random. Then, $\mathbb{E}_{B \sim U(b_1, \ldots, b_m)} c_{ij} = a_i b_j/ n.$
\end{lemma}
\begin{proof}
To prove this equality, we simply note that 
\[
    \mathbb{E}_{B \sim U(b_1, \ldots, b_m)}c_{ij} = \sum_{x \in A_i}\mathbb{E}\ \1\{x \in B_j\} = a_i \, \mathbb{P} \left(\tilde{x} \in B_j\right) = a_i \, \mathbb{E} \sum_{y \in B_j} \1 \{\tilde{x} = y\} =  \frac{a_i b_j}{n},
\]
where $\tilde{x}$ is an arbitrary element of $A_i$.
\end{proof}

Now, let us prove that all measures that have the exact constant baseline property also have the approximate constant baseline.

\begin{lemma}
\label{lem:approximate-constant-baseline}
If a measure $M\left(\mathcal{C}\right)$ is scale-invariant (see Definition~\ref{def:scale-invariance}), continuous, and has the constant baseline property, then it also has the approximate constant baseline.
\end{lemma}
\begin{proof}
Let us fix non-negative numbers $\{a_i\}_{i=0}^{m-1}, \{b_i\}_{i=0}^{m-1}$ such that $\sum_{i=0}^{m-1} a_i = \sum_{i=0}^{m-1} b_i = n$. Then, consider a fixed classification $A^N$ with class sizes $N a_1, \ldots, N a_m$ and a random classification $B^N$ taken from $U(N b_1, \ldots, N b_m)$.

Let $c_{ij}^N$ denote entries of the confusion matrix for $A^N$ and $B^N$. Let us prove that for any $i, j \in \{1, \ldots m\}$, the random variable $c_{ij}^N / N$ converges to $a_i b_j / n$ in $L_2$ as $N \rightarrow \infty$. From Lemma~\ref{lem:constant-baseline-expectation}, we have $\mathbb{E}\left(c_{ij} / N\right) = a_i b_j / n$.
Let us compute $\Var\left(c_{ij}\right)$.
Recall that $c_{ij} = \sum_{x \in A_i^N} \1\{x \in B_j^N\}$, then
\[
    \Var\left(c_{ij}\right) = \sum_{x, y \in A_i^N}\Cov\left(\1\{x \in B_j^N\}, \1\{y \in B_j^N\}\right).
\]
It remains to compute $\Cov\left(\1\{x \in B_j^N\}, \1\{y \in B_j^N\}\right)$ for $x = y$ and $x \neq y$.
For this, note that 
\[  
\mathbb{P}\left(x \in B_j^N\right) = b_j / n\ \text{ and }\    \mathbb{P}\left(x, y \in B_j^N\right) = \frac{N b_j (N b_j -1)}{Nn (N n - 1)}\ \text{for}\ x \neq y.
\]
Then,
\[  
\Cov\left(\1\{x \in B_j^N\}, \1\{y \in B_j^N\}\right) = \mathbb{P}\left(x, y \in B_j^N\right) - \left(\mathbb{P}\left(x \in B_j^N\right)\right)^2 = O(1 / N).
\]
Thus, we get that $\Var\left(c_{ij} / N\right) = O(N) / N^2 = O(1 / N)$ and prove $L_2$-convergence.

Now we are ready to prove the lemma. Let $M$ be a scale-invariant, continuous measure that has constant baseline. Then, 
\[
c_{\text{base}} = \mathbb{E}_{B^N \sim U(N b_1, \ldots, N b_m)} M\left(\mathcal{C}_N\right) = \mathbb{E}M\left(\frac{\mathcal{C}^N}{N}\right) \xrightarrow[N \rightarrow \infty]{} M\left(\mathcal{C}\right),
\] where $\mathcal{C}^N$ is the confusion matrix for $A^N$ and $B^N$ and $\mathcal{C}$ is the confusion matrix for $A$ and $B$. Here $\mathbb{E}M\left(\mathcal{C}^N/N\right) \to M\left(\mathcal{C}\right)$ holds since the $L_2$-convergence of $c_{ij}^N$ to $a_i b_j/n$ implies convergence in distribution. 
\end{proof}

From this lemma, we get that $F_1$, Jaccard, Acc, and CE do not have constant baseline since they violate the approximate constant baseline property.

Assume that a measure $M\left(\mathcal{C}\right)$ is linear in $c_{ii}$ for fixed $a_i$ and $b_j$. Then, using the linearity of expectation, we note that approximate constant baseline implies exact constant baseline for such measures. This observation gives that CC, BA, $\kappa$, SBA, and GM$_r$ have the constant baseline property.

Finally, we note that CD violates the constant baseline property as it has both monotonicity and distance properties (in binary case), while Theorem~\ref{thm:impossibility} states that all three properties cannot be simultaneously satisfied.

\subsection{Preserving properties by averagings}

\paragraph{Micro averaging}

Recall that for micro averaging, we sum up the binary confusion matrices corresponding to $m$ one-vs-all classifications. Formally, we set $\text{TP} := \sum_{i=0}^{m-1} c_{ii}$, $\text{FN} := \text{FP} = n - \sum_{i=0}^{m-1} c_{ii}$, $\text{TN} := (m-2) n + \sum_{i=0}^{m-1} c_{ii}$. Then, we compute the binary measure.

First, it is easy to see that this averaging preserves symmetry and class-symmetry. 

Let us prove that micro averaging preserves the maximal agreement property. If a confusion matrix $\mathcal{C}$ is diagonal, then $n - \sum_{i=0}^{m-1} c_{ii} = 0$ and $\text{FP} = \text{FN} = 0$.
Substituting these values in a binary measure $M$, we get $c_{\max}$. If $\mathcal{C}$ is not diagonal, then $\text{FP} = \text{FN} = n - \sum_{i=0}^{m-1} c_{ii} > 0$ and the result of the averaging will be strictly lower than $c_{\max}$. On the other hand, minimal agreement is not preserved since $\text{TN} = (m - 2) n > 0$ for zero-diagonal confusion matrices. As a simple example, consider a measure $\1\{\text{TP} + \text{TN} > 0\}$ satisfying the minimal agreement property. Then, after micro averaging, this measure is constant, thus violating minimal agreement.

Also, micro averaging preserves monotonicity: increasing $c_{ii}$ for fixed $n$ leads to increased $\text{TP}$ and $\text{TN}$, leaving $\text{TP} + \text{FP}, \text{TP} + \text{FN}, \text{TN} + \text{FP}, \text{TN} + \text{FN}$ unchanged. On the other hand, strong monotonicity can be violated: incrementing $c_{ij}$ for $i \neq j$ we increase $n$, so $\text{TN} = (m - 2) n + \sum_{i=0}^{m-1} c_{ii}$ also increases and the averaged measure may increase.
For example, consider a strongly monotone binary measure $\text{TP} + \text{TN} - \text{FP} - \text{FN}$.
Then, after micro averaging, it reduces to $n m$, which violates strong monotonicity.

To prove that micro averaging preserves the distance property, we first note that it preserves maximal agreement and symmetry. To show that the triangle inequality is also preserved, we consider micro averaging as a result of the following procedure. First, we use one-hot encoding to map each class to a binary vector. Then, we map a classification vector $A$ of size $n$ to the binary vector $\hat{A}$ of size $nm$ consisting of one-hot encoded binary vectors. Finally, for two classifications $A$ and $B$, we compute the binary measure for $\hat{A}$ and $\hat{B}$. It is easy to see that this procedure is equivalent to micro averaging. Thus, for any multiclass labelings $A, B, C$, there exist binary labelings $\hat{A}, \hat{B}, \hat{C}$ with confusion matrices corresponding to the result of micro averaging. Hence, the triangle inequality for micro averaged matrices follows from the binary property.

Finally, approximate constant baseline can be violated after micro averaging. Indeed, let us take $c_{ii} = {a_i b_i} / {n}$. Then, after the averaging, we get $\text{TP} = \sum_{i=0}^{m-1} {a_i b_i} / {n}$, which is not necessary equal to ${(\text{TP} + \text{FN}) (\text{TP} + \text{FP})} / (m n) = {n} / {m}$. As an example, we can consider a measure $\text{TP} - (\text{TP} + \text{FP}) (\text{TP} + \text{FN}) / (\text{TP} + \text{FP} + \text{FN} + \text{TN})$ having constant baseline. Thus, the averaged measure is $\sum_{i=0}^{m-1} c_{ii} - n / m$, which does not have an approximate constant baseline. Consequently, the constant baseline property is also violated.

\paragraph{Macro averaging}

As for the micro averaging, symmetry and class-symmetry are clearly satisfied.

Let us check the maximal agreement. Consider a binary measure $M$ having this property. If $\mathcal{C}$ is diagonal, then the result of the averaging is $\frac{1}{m}\sum_i M(c_{ii}, 0, 0, n - c_{ii}) = c_{\max}$. If $\mathcal{C}$ is not diagonal, then one of $a_i - c_{ii} > 0$ and the averaged measure is strictly lower than $c_{\max}$. In contrast, the minimal agreement property can be violated, since for a zero-diagonal confusion matrix the result of the averaging is $\frac{1}{m}\sum_i M(0, a_i, b_i, n - a_i - b_i)$. Since we may have $n- a_i - b_i > 0$, the minimal agreement can be violated. For instance, consider the measure $\1\{\text{TP} + \text{TN} > 0\}$ satisfying the minimal agreement property in the binary case. Then, taking
$
\mathcal{C}_1 = 
\left(\begin{smallmatrix}
    0 & 1 & 0\\
    0 & 0 & 1\\
    1 & 1 & 0
\end{smallmatrix}\right)
$
and
$
\mathcal{C}_2 = 
\left(\begin{smallmatrix}
    0 & 0 & 1\\
    0 & 0 & 1\\
    1 & 1 & 0
\end{smallmatrix}\right)
$ we get that the averaging has different values on these matrices (1 and $\nicefrac{2}{3}$, respectively), thus the minimal agreement property does not hold.

It is easy to see that monotonicity is preserved by macro averaging. However, strong monotonicity can be violated. Indeed, assume that $c_{ij}$ increases. Then, for $k \notin \{i, j\}$, the values $c_{kk}$, $a_k$, $b_k$ do not change while $n$ increases. To show that this can break strong monotonicity, consider the same counterexample as for the micro averaging: $\text{TP} + \text{TN} - \text{FP} - \text{FN}$. Then, after macro averaging, we get the measure $\left(n (m - 4) + 4\sum_{i=0}^{m-1} c_{ii}\right) / m$ that is not strongly monotone.

Let us prove that macro averaging preserves the distance property. As for the micro averaging, it remains to check the triangle inequality. Let $A$, $B$, and $C$ be multiclass classifications with $n$ elements and $m$ classes. Then, for all $i \in \{1, \ldots, m\}$, we can build the binary labelings $A^i, B^i, C^i$ corresponding to one-vs-all binary classifications. Triangle inequality holds for each $A^i, B^i, C^i$. Thus, summing up these inequalities over all $i \in \{1, \ldots, m\}$, we prove the triangle inequality for the macro-averaged measure.

Finally, approximate and exact constant baseline are preserved by the macro averaging due to the linearity of expectation.

\paragraph{Weighted averaging}

Similar reasoning as above, allows one to show that weighted averaging preserves the maximal agreement, class-symmetry, monotonicity, exact and approximate constant baseline.

For the minimal agreement, the counterexample used for macro averaging also works in this case.

Clearly, weighted averaging is not symmetric: we normalize by the class sizes $a_i$. Therefore, the distance property is not preserved as it requires symmetry.

Finally, as a counterexample to strong monotonicity, we can take $M = \text{TP} + \text{TN} - \text{FP} - \text{FN}$ and
$
\mathcal{C}_1 = 
\left(\begin{smallmatrix}
    0 & 1 & 1\\
    1 & 0 & 0\\
    1 & 0 & 0   
\end{smallmatrix}\right)
$,
$
\mathcal{C}_2 = 
\left(\begin{smallmatrix}
    0 & 1 & 1\\
    1 & 0 & 1\\
    1 & 0 & 0   
\end{smallmatrix}\right)
$,
Then, $M(\mathcal{C}_1) = -2 < -\nicefrac{9}{5} = \mathcal{C}_2)$.

\section{Theoretical analysis}\label{app:analysis}

In this section, we perform a theoretical analysis of binary classification measures. First, we generalize the definition of constant baseline and theoretically compare the two non-linear distance-transformations of the Matthews Correlation Coefficient. Then, we derive the class of measures that satisfy all properties except distance.

\subsection{Higher-order approximate constant baseline}
Before we generalize our definition of constant baseline, let us introduce some additional properties. These properties differ from the properties introduced in the main text in the sense that they are not desirable in themselves but are rather \emph{instrumental} for the analysis of other desirable properties.

\begin{definition}
\label{def:scale-invariance}
A measure $M$ is \emph{scale-invariant} if, for any scalar $\alpha>0$ and confusion matrix $\mathcal{C}$, $M(\alpha\mathcal{C})=M(\mathcal{C})$.
\end{definition}

We remark that all measures of Table~\ref{tab:validation_indices} are scale-invariant.

Note that any binary classification measure can be written as a function of the four variables $c_{11}$, $a_1$, $b_1$, $n$ as $c_{10}=a_1-c_{11}$, $c_{01}=b_1-c_{11}$, and $c_{00}=n-a_1-b_1+c_{11}$. Therefore, any scale-invariant binary classification measure can be written as a function of the three fractions $p_{AB}=c_{11}/n$, $p_A=a_1/n$, and $p_B=b_1/n$. Hence, we will use the shorthand notation $M(\mathcal{C})=M(p_{AB},p_A,p_B)$ for the remainder of this analysis. We will write $P_{AB}$ instead of $p_{AB}$ whenever $B$ is random. Note that for $B\sim U(p_Bn,(1-p_B)n)$, it holds that $\mathbb{E}_{B\sim U(p_Bn,(1-p_B)n)}[P_{AB}]=p_Ap_B$. Thus, it can readily be seen that the approximate constant baseline is satisfied whenever $M(p_Ap_B,p_A,p_B)=c_{\text{base}}$. We introduce one additional property that ensures that the measure is a well-behaved function in terms of these variables.

\begin{definition}
A scale-invariant measure $M$ is \emph{smooth} if, for any $p_A,p_B\in(0,1)$, the Taylor series of $M(p_{AB},p_A,p_B)$ around the point $p_{AB}=p_Ap_B$ converges absolutely on the interval $p_{AB}\in[0,\min\{p_A,p_B\}]$. That is, for all $p_A,p_B\in(0,1)$ and $p_{AB}\in[0,\min\{p_A,p_B\}]$, we have
\[
\sum_{k=0}^{\infty}\left|\frac{(p_{AB}-p_Ap_B)^k}{k!}\frac{\partial^k}{\partial p_{AB}^k}M(p_Ap_B,p_A,p_B)\right|<\infty.
\]
\end{definition}
Note that such absolute convergence implies that the Taylor series converges to $M(p_{AB},p_A,p_B)$. We remark that all constant-baseline measures of Table~\ref{tab:validation_indices} are linear functions in $p_{AB}$ for fixed $p_A,p_B$. Thus, each of these is smooth. Furthermore, because CC is linear in $p_{AB}$, we have that for any transformation $f(\text{CC})$, the Taylor expansion of $f(\text{CC})$ is given by substituting CC in the Taylor expansion of $f$. Thus, since the Taylor expansion of $f_1(x)=\tfrac{1}{\pi}\arccos(x)$ and $f_2(x)=\sqrt{2(1-x)}$ around $x=0$ converges for $x\in [-1,1]$, we have that CD$=f_1(\text{CC})$ and CD$'=f_2(\text{CC})$ are also smooth measures.

This allows us to express the expected value of a measure in terms of the central moments of $P_{AB}$:
\begin{align*}
    \mathbb{E}[M(P_{AB},p_A,p_B)]&=\mathbb{E}\left[
    \sum_{k=0}^{\infty}\frac{(P_{AB}-p_Ap_B)^k}{k!}\frac{\partial^k}{\partial p_{AB}^k}M(p_Ap_B,p_A,p_B)
    \right]\\
    &=\sum_{k=0}^{\infty}\frac{\mathbb{E}[(P_{AB}-p_Ap_B)^k]}{k!}\frac{\partial^k}{\partial p_{AB}^k}M(p_Ap_B,p_A,p_B).
\end{align*}
Here, the absolute convergence helps bound the term inside the expectation so that the Dominated Convergence Theorem allows us to interchange summation and expectation. In this expression, the first-order term vanishes as $\mathbb{E}[P_{AB}]=p_Ap_B$. Thus, we have
\[
\mathbb{E}[M(P_{AB},p_A,p_B)]=M(p_Ap_B,p_A,p_B)+\sum_{k=2}^{\infty}\frac{\mathbb{E}[(P_{AB}-p_Ap_B)^k]}{k!}\frac{\partial^k}{\partial p_{AB}^k}M(p_Ap_B,p_A,p_B).
\]
Note that for large numbers of items, $P_{AB}$ is highly concentrated around $p_Ap_B$. Thus, the contribution of the higher-order central moments is relatively small. This leads to the following generalization of the constant baseline.
\begin{definition}
A smooth measure $M$ has a \emph{$k$-th order approximate constant baseline}, if there exists a constant $c_{\text{base}}$ such that $M(p_Ap_B,p_A,p_B)=c_{\text{base}}$, while for all $\ell\in\{2,\dots,k\}$, it holds that
\[
\frac{\partial^\ell}{\partial p_{AB}^\ell}M(p_Ap_B,p_A,p_B)=0.
\]
\end{definition}
Thus, first-order constant baseline is equivalent to the approximate constant baseline. Furthermore, note that $\infty$-th order approximate constant baseline implies exact constant baseline since then
\[
\mathbb{E}[M(P_{AB},p_A,p_B)]=M(p_Ap_B,p_A,p_B)=c_{\text{base}}.
\]
While it seems likely that the exact constant baseline also implies $\infty$-th order constant baseline, we were not able to formally prove this. However, all constant-baseline measures of Table~\ref{tab:validation_indices} also satisfy $\infty$-th order constant baseline. For this reason, we will use $\infty$-th order constant baseline as a substitute for the exact constant baseline when deriving measures from properties.

\subsection{Constant baseline order of distance transformations}\label{app:cb_distance_order}
We now show that the constant baseline of $\text{CD}=\tfrac{1}{\pi}\arccos(\text{CC})$ is one order higher than $\text{CD}'=\sqrt{2(1-\text{CC})}$.

\begin{statement}\label{st:CD_order}
$\emph{\text{CD}}=\tfrac{1}{\pi}\arccos(\emph{\text{CC}})$ has a second-order approximate constant baseline while $\emph{\text{CD}}'=\sqrt{2(1-\emph{\text{CC}})}$ only has a first-order approximate constant baseline.
\end{statement}
\begin{proof}
The Matthews Correlation Coefficient is given by
\[
\text{CC}(p_{AB},p_A,p_B)=\frac{p_{AB}-p_Ap_B}{\sqrt{p_A(1-p_A)p_B(1-p_B)}},
\]
so that it is indeed a linear function in $p_{AB}$ for fixed $p_A,p_B$.
Therefore, the Taylor expansions of CD and CD$'$ are obtained by simply substituting CC into the Taylor expansions of $\tfrac{1}{\pi}\arccos(x)$ and $\sqrt{2(1-x)}$ respectively. We have
\[
\tfrac{1}{\pi}\arccos(x)=\frac{\pi}{2}-\sum_{k=0}^\infty\frac{(2k)!x^{2k+1}}{4^k(k!)^2(2k+1)}\,\text{ and } \sqrt{2(1-x)}=\sqrt{2}-\sqrt{2}\sum_{k=0}^\infty\frac{2}{k+1}\binom{2k}{k}\left(\frac{x}{4}\right)^{k+1}.
\]
Thus, we see that $\sqrt{2(1-x)}$ we have a quadratic term, which we do not have for $\tfrac{1}{\pi}\arccos(x)$. This shows that CD has a second-order constant baseline while CD$'$ only has a first-order constant baseline.
\end{proof}

\subsection{Deriving measures satisfying all properties except distance}

Let us derive a class of measures satisfying all properties from Table~\ref{tab:properties} except distance.
We will use $\infty$-th order constant baseline instead of the exact constant baseline as this property is easier to analyze while it implies exact constant baseline and coincides with it for all measures of Table~\ref{tab:validation_indices}.

\begin{theorem}
\label{thm:cb-class}
Let $M$ be a smooth binary classification measure that satisfies the following properties:
\begin{enumerate}
    \item $\infty$-th order constant baseline with constant $0$;
    \item symmetry;
    \item class-symmetry;
    \item maximal agreement with constant $1$;
    \item minimal agreement with constant $-1$;
    \item strong monotonicity.
\end{enumerate}
Then, it is of the following form:
\[
M(p_{AB},p_A,p_B)=s(p_A,p_B)(p_{AB}-p_Ap_B),
\]
where $s$ satisfies the following properties:
\begin{enumerate}
\item $s(p_B,p_A)=s(p_A,p_B)=s(1-p_A,1-p_B)$;
\item $s(p_A,p_A)=s(p_A,1-p_A)=\frac{1}{p_A(1-p_A)}$;
\item $s(p_A,p_B)<\max\left\{\frac{1}{p_Ap_B},\frac{1}{(1-p_A)(1-p_B)}\right\}$ for~$p_B\neq 1-p_A$;
\item $s(p_A,p_B)<\max\left\{\frac{1}{p_A(1-p_B)},\frac{1}{(1-p_A)p_B}\right\}$ for~$p_B\neq p_A$;
\item $\frac{1}{s}\left(p_{A}\frac{\partial}{\partial p_{A}}+p_{B}\frac{\partial}{\partial p_{B}}\right)s\in\left[\min\left\{-2,-1-\frac{p_Ap_B}{(1-p_A)(1-p_B)}\right\},\max\left\{\frac{2p_B-1}{1-p_B},\frac{2p_A-1}{1-p_A}\right\}\right]$;
\item $\frac{1}{s}\left((1-p_{A})\frac{\partial}{\partial p_{A}}-p_{B}\frac{\partial}{\partial p_{B}}\right)s\in\left[\min\left\{2-\frac{1}{p_A},2-\frac{1}{1-p_B}\right\},\max\left\{1+\frac{p_B(1-p_A)}{p_A(1-p_B)},2\right\} \right]$.
\end{enumerate}
\end{theorem}

\begin{proof}
From the definition of $\infty$-th order constant baseline, we have that $M(p_{AB},p_A,p_B)$ must be a linear function in $p_{AB}$ for fixed $p_A,p_B$. Thus, it must be of the form
\[
M(p_{AB},p_A,p_B)=c_{\text{base}}+(p_{AB}-p_Ap_B)s(p_A,p_B)=(p_{AB}-p_Ap_B)s(p_A,p_B)
\]
for some function $s(\cdot,\cdot)$. 

Now, symmetry requires $M(p_{AB},p_B,p_A)=M(p_{AB},p_A,p_B)$ which leads to $s(p_B,p_A)=s(p_A,p_B)$. Then, class-symmetry requires $M(p_{AB},p_A,p_B) = M(1-p_A-p_B+p_{AB},1-p_A,1-p_B)$, leading to $s(1-p_A,1-p_B)=s(p_A,p_B)$.

For maximal agreement, we have $M(p_{AB},p_A,p_B)\leq 1$ with equality only if $p_{AB}=p_A=p_B$, i.e., 
$M(p_A,p_A,p_A)=1$, leading to $s(p_A,p_A)=\frac{1}{p_A(1-p_A)}$. Furthermore, $M(p_{AB},p_A,p_B)\leq M(\min\{p_A,p_B\},p_A,p_B)<1$ for $p_A\neq p_B$ is satisfied by
\begin{align*}
s(p_A,p_B)<\frac{1}{\min\{p_A,p_B\}-p_Ap_B}&=\frac{1}{\min\{p_A(1-p_B),(1-p_A)p_B\}}\\
&=\max\left\{\frac{1}{p_A(1-p_B)},\frac{1}{(1-p_A)p_B}\right\}.
\end{align*}
Minimal agreement requires $M(p_{AB},p_A,p_B)\geq -1$ with equality only if $p_{AB}=0,p_B=1-p_A$. For equality, we need 
\[
s(p_A,1-p_A)=\frac{1}{p_A(1-p_A)}.
\]
While for $p_B\neq 1-p_A$, we need $M(p_{AB},p_A,p_B)\geq M(\max\{0,p_A+p_B-1\},p_A,p_B)>-1$, leading to
\[
s(p_A,p_B)<\frac{1}{\min\{p_Ap_B,(1-p_A)(1-p_B)\}}=\max\left\{\frac{1}{p_Ap_B},\frac{1}{(1-p_A)(1-p_B)}\right\}.
\]
For the remainder of the proof, we will derive that strong monotonicity is satisfied when the last two conditions of Theorem~\ref{thm:cb-class} hold. The first one will be derived from the increasingness of $M$ in $N_{00}$ while the second one will be derived from decreasingness in $N_{10}$. Increasingness in $N_{11}$ and decreasingness in $N_{01}$ will then follow from class-symmetry and symmetry, respectively.

We rewrite the condition $\frac{d}{dN_{00}}M$ to
\begin{align*}
    &\frac{d}{dN_{00}}M\left(\frac{N_{11}}{N_{11}+N_{10}+N_{01}+N_{00}},\frac{N_{11}+N_{10}}{N_{11}+N_{10}+N_{01}+N_{00}},\frac{N_{11}+N_{01}}{N_{11}+N_{10}+N_{01}+N_{00}}\right)\\
=&-\frac{1}{N}\left[p_{AB}\frac{\partial}{\partial p_{AB}}+p_{A}\frac{\partial}{\partial p_{A}}+p_{B}\frac{\partial}{\partial p_{B}}\right]M(p_{AB},p_A,p_B).
\end{align*}
Since we want $\frac{d}{dN_{00}}M>0$, we need
\[
\left[p_{AB}\frac{\partial}{\partial p_{AB}}+p_{A}\frac{\partial}{\partial p_{A}}+p_{B}\frac{\partial}{\partial p_{B}}\right]M(p_{AB},p_A,p_B)<0.
\]
We compute the partial derivatives of $M$:
\begin{align}
\label{eq:strong_monotonicity_partials}
\begin{split}
    &\frac{\partial}{\partial p_{AB}}M=s,\\ 
    &\frac{\partial}{\partial p_{A}}M=-p_Bs+(p_{AB}-p_Ap_B)\frac{\partial}{\partial p_{A}}s,\\
    &\frac{\partial}{\partial p_{B}}M=-p_As+(p_{AB}-p_Ap_B)\frac{\partial}{\partial p_{B}}s.
\end{split}
\end{align}
Thus, we need
\[
(p_{AB}-2p_Ap_B)\cdot s+(p_{AB}-p_Ap_B)\left[p_{A}\frac{\partial}{\partial p_{A}}+p_{B}\frac{\partial}{\partial p_{B}}\right]s<0
\]
for all $p_{AB}\in\left[\max\{p_A+p_B-1,0\},\min\{p_A,p_B\}\right]$. Since the left-hand side is linear in $p_{AB}$, we only need to check the upper and lower limit. Substituting $p_{AB}=\min\{p_A,p_B\}$ leads to
\begin{align*}
    \left[p_{A}\frac{\partial}{\partial p_{A}}+p_{B}\frac{\partial}{\partial p_{B}}\right]s&<\frac{2p_Ap_B-\min\{p_A,p_B\}}{\min\{p_A,p_B\}-p_Ap_B}s\\
    &=\left(\frac{p_Ap_B}{\min\{p_A(1-p_B),p_B(1-p_A)\}}-1\right)s\\
    &=\max\left\{\frac{p_B}{1-p_B}-1,\frac{p_A}{1-p_A}-1\right\}s\\
    &=\max\left\{\frac{2p_B-1}{1-p_B},\frac{2p_A-1}{1-p_A}\right\}s.
\end{align*}
Substituting $p_{AB}=\max\{0,p_A+p_B-1\}$ gives a lower bound
\begin{align*}
    \left[p_{A}\frac{\partial}{\partial p_{A}}+p_{B}\frac{\partial}{\partial p_{B}}\right]s
    &>-\frac{2p_Ap_B-\max\{0,p_A+p_B-1\}}{p_Ap_B-\max\{0,p_A+p_B-1\}}s\\
    &=-\left(1+\frac{p_Ap_B}{\min\{p_Ap_B,(1-p_A)(1-p_B)\}}\right)s\\
    &=-\max\left\{2,1+\frac{p_Ap_B}{(1-p_A)(1-p_B)}\right\} s.
\end{align*}
Combining this, we conclude that increasingness in $N_{00}$ is satisfied whenever it holds that
\begin{align*}
    \frac{1}{s}\left(p_{A}\frac{\partial}{\partial p_{A}}+p_{B}\frac{\partial}{\partial p_{B}}\right)s\in\left[\min\left\{-2,-1-\frac{p_Ap_B}{(1-p_A)(1-p_B)}\right\},\max\left\{\frac{2p_B-1}{1-p_B},\frac{2p_A-1}{1-p_A}\right\}\right],
\end{align*}
as required. 

The condition for decreasingness in $N_{10}$ is obtained similarly. The condition $\frac{d}{dN_{10}}M<0$ can be rewritten to
\[
\left[-p_{AB}\frac{\partial}{\partial p_{AB}}+(1-p_{A})\frac{\partial}{\partial p_{A}}-p_{B}\frac{\partial}{\partial p_{B}}\right]M(p_{AB},p_A,p_B)<0.
\]
Substituting the partial derivatives from~\eqref{eq:strong_monotonicity_partials} gives
\[
s\cdot(-p_{AB}-(1-p_A)p_B+p_Ap_B)+(p_{AB}-p_Ap_B)\left((1-p_{A})\frac{\partial}{\partial p_{A}}-p_{B}\frac{\partial}{\partial p_{B}}\right)s<0.
\]
Again, this linear inequality should hold for all $p_{AB}\in\left[\max\{p_A+p_B-1,0\},\min\{p_A,p_B\}\right]$ and we only need to test the extremes. For $p_{AB}=\min\{p_A,p_B\}$, we find the upper bound
\begin{align*}
    \frac{1}{s}\left((1-p_{A})\frac{\partial}{\partial p_{A}}-p_{B}\frac{\partial}{\partial p_{B}}\right)s
&<\frac{\min\{p_A,p_B\}+(1-p_A)p_B-p_Ap_B)}{\min\{p_A,p_B\}-p_Ap_B}\\
&=\frac{\min\{p_A(1-p_B)+p_B(1-p_A),2p_B(1-p_A)\}}{\min\{p_A(1-p_B),p_B(1-p_A)\}}\\
&=\max\left\{1+\frac{p_B(1-p_A)}{p_A(1-p_B)},2\right\}.
\end{align*}
Substituting  $p_{AB}=\max\{0,p_A+p_B-1\}$, we get the following upper bound
\begin{align*}
    \frac{1}{s}\left((1-p_{A})\frac{\partial}{\partial p_{A}}-p_{B}\frac{\partial}{\partial p_{B}}\right)s
&>\frac{\max\{0,p_A+p_B-1\}+(1-p_A)p_B-p_Ap_B)}{\max\{0,p_A+p_B-1\}-p_Ap_B}\\
&=-\frac{\max\{p_B(1-2p_A),p_A+2p_B-1-2p_Ap_B\}}{\min\{p_Ap_B,(1-p_A)(1-p_B)\}}\\
&=\min\left\{2-\frac{1}{p_A},2-\frac{1}{1-p_B}\right\}.
\end{align*}
Combined, we obtain the desired condition
\[
\frac{1}{s}\left((1-p_{A})\frac{\partial}{\partial p_{A}}-p_{B}\frac{\partial}{\partial p_{B}}\right)s\in\left[\min\left\{2-\frac{1}{p_A},2-\frac{1}{1-p_B}\right\},\max\left\{1+\frac{p_B(1-p_A)}{p_A(1-p_B)},2\right\}\right].
\]
\end{proof}

\subsection{Generalized Means measure}

The Generalized Means measure GM$_r$ corresponds to $s(p_A,p_B)=M_r(p_A(1-p_A),p_B(1-p_B))^{-1}$, where $M_r$ is the generalized mean with exponent $r$. 

\begin{lemma}
$s(p_A,p_B) = M_r(p_A(1-p_A),p_B(1-p_B))^{-1}$ satisfies all the conditions of Theorem~\ref{thm:cb-class}.
\end{lemma}

\begin{proof}
The proof follows from Section~\ref{sec:checking-properties}, where it is shown that GM$_r$ indeed satisfies all the required properties. Let us also demonstrate the conditions explicitly.

The first four conditions can be easily verified by substituting this $s(p_A,p_B)$. Verifying the last two conditions require a bit more work. The partial derivatives of $s(p_A,p_B)$ are given by
\begin{align*}
    &\frac{\partial}{\partial p_A}\left[\frac{1}{2}\left(p_A(1-p_A)\right)^r+\frac{1}{2}\left(p_B(1-p_B)\right)^r\right]^{-\frac{1}{r}}\\
    =&-\frac{1}{r}\frac{\frac{r}{2}\left(p_A(1-p_A)\right)^{r-1}(1-2p_A)}{\left[\frac{1}{2}\left(p_A(1-p_A)\right)^r+\frac{1}{2}\left(p_B(1-p_B)\right)^r\right]^{\frac{r+1}{r}}}\\
    &=\frac{2p_A-1}{p_A(1-p_A)}\cdot\frac{\left(p_A(1-p_A)\right)^r}{\left(p_A(1-p_A)\right)^r+\left(p_B(1-p_B)\right)^r}\cdot s,
\end{align*}
and similarly
\[
\frac{\partial}{\partial p_B}s=\frac{2p_B-1}{p_B(1-p_B)}\cdot\frac{\left(p_B(1-p_B)\right)^r}{\left(p_A(1-p_A)\right)^r+\left(p_B(1-p_B)\right)^r}\cdot s.
\]

Substituting this into the condition for $N_{00}$-monotonicity, we get
\begin{multline*}
    \frac{1}{s}\left(p_{A}\frac{\partial}{\partial p_{A}}+p_{B}\cdot\frac{\partial}{\partial p_{B}}\right)s\\
    =\frac{2p_A-1}{1-p_A}\cdot\frac{\left(p_A(1-p_A)\right)^r}{\left(p_A(1-p_A)\right)^r+\left(p_B(1-p_B)\right)^r}+\frac{2p_B-1}{1-p_B}\cdot\frac{\left(p_B(1-p_B)\right)^r}{\left(p_A(1-p_A)\right)^r+\left(p_B(1-p_B)\right)^r}.
\end{multline*}
Note that the two large fractions sum to $1$, so that we recognize this as the weighted average of $(2p_A-1)/(1-p_A)$ and $(2p_B-1)/(1-p_B)$, which are exactly the two terms in the maximum of the upper bound of the $N_{00}$-monotonicity condition. Furthermore, note that both these terms are larger than $-1$, so that the lower bound is also satisfied.

Similarly, for the condition corresponding to $N_{10}$-monotonicity, we get
\begin{multline*}
    \frac{1}{s}\left((1-p_{A})\frac{\partial}{\partial p_{A}}-p_{B}\frac{\partial}{\partial p_{B}}\right)s\\
    =\left(2-\frac{1}{p_A}\right)\frac{\left(p_A(1-p_A)\right)^r}{\left(p_A(1{-}p_A)\right)^r+\frac{1}{2}\left(p_B(1{-}p_B)\right)^r}+\left(2-\frac{1}{1{-}p_B}\right)\frac{\left(p_B(1-p_B)\right)^r}{\left(p_A(1{-}p_A)\right)^r+\left(p_B(1{-}p_B)\right)^r}.
\end{multline*}
Again, we recognize this as the weighted average of $2-p_A^{-1}$ and $2-(1-p_B)^{-1}$, which are the terms in the minimum of the required lower bound, so that this is always satisfied. Finally, the corresponding upper bound is always satisfied since $2-p_A^{-1}$ and $2-(1-p_B)^{-1}$ can both be upper-bounded by $1$. We thus conclude that GM$_r$ indeed lies inside this class of measures for all $r$.
\end{proof}

\paragraph{Proof of Statement~\ref{statement2}}

Finally, let us show that Generalized Means generalizes both CC and SBA.
Recall that 
\[
\text{GM}_r = \frac{n c_{11} - a_1 b_1}{\left(\frac{1}{2} \left(a_1^r a_0^r + b_1^r b_0^r\right)\right)^{\frac{1}{r}}}.
\]
Taking $r = -1$, we obtain:
\begin{align*}
    1 + \text{GM}_{-1} 
    &= 1 + \frac{1}{2}\left(\frac{n c_{11}}{a_0 a_1} + \frac{n c_{11}}{b_0 b_1} - \frac{b_1}{a_0} - \frac{a_1}{b_0}\right) \\ 
    &= \frac{1}{2}\left(\frac{c_{11} (a_0 + a_1)}{a_0 a_ 1} + \frac{c_{11} (b_0 + b_1)}{b_0 b_1} - \frac{b_1}{a_0} - \frac{a_1}{b_0} + 2\right)\\
    &= \frac{1}{2} \left(\frac{c_{11}}{a_1} + \frac{c_{11}}{b_1} + \frac{n - a_1 - b_1 + c_{11}}{a_0} + \frac{n - a_1 - b_1 + c_{11}}{b_0}\right) \\
    &= 2 \cdot \text{SBA}\,.
\end{align*}

Now, let us confirm that taking $r \rightarrow 0$ we get CC.
Let $X := b_0 b_1 / (a_0 a_1)$, then $\left(\frac{1}{2} (a_1^r a_0^r + b_1^r b_0^r)\right)^{\frac{1}{r}}$ can be rewritten to 
\[
a_0a_1 \left(\frac{1}{2}\left(1 + X^r\right)\right)^{\frac{1}{r}}
=a_0a_1\exp\left(\frac{1}{r}\ln \left(\frac{1}{2}\left(1 + X^r\right)\right)\right).
\]
We take the limit of the exponent and use l'H\^opital to find that
\[
    \lim_{r\rightarrow0}\frac{\ln\left(\frac{1}{2}\left(1 + X^r\right)\right)}{r} 
    =\lim_{r\rightarrow0}\frac{\ln(X)X^r}{1+X^r}=\frac{1}{2}\ln X.
\]
Hence, as $r \to 0$, the denominator of GM$_r$ converges to
\[
a_0 a_1 \cdot \exp\left(\frac{1}{2}\ln X\right) = a_0 a_1 \sqrt{X} = \sqrt{a_0 a_1 b_0 b_1}
\]
and we obtain CC.

\section{Additional experimental results}\label{app:experiments}

\subsection{Binary measures}\label{app:binary}

\paragraph{Distinguishing binary measures}

Let us show triplets of labelings $(A, B_1, B_2)$ discriminating all pairs of measures in the binary classification case. Each triplet consists of the true labeling $A$ and two predicted labelings $B_1$ and $B_2$. We say that two measures are strictly inconsistent if, according to the first one, $B_1$ is closer to $A$, while, according to the second one, $B_2$ is closer to $A$ (comparing to the main text, here we consider only strict inequalities). Table~\ref{tab:synthetic_binary} lists six triplets, where all labelings are of size $n = 10$. It also specifies which triplet discriminates each pair of measures.

\begin{table}
\caption{Examples of triplets discriminating all pairs of different measures: the upper table lists the triplets, the lower table specifies which triplet discriminates a particular pair}
\label{tab:synthetic_binary}
\vspace{4pt}
\centering
\begin{tabular}{l|ccc}
\toprule
 & $A$ & $B_1$ & $B_2$ \\
\midrule
Triplet 1 & (1, 1, 1, 0, 1, 1, 0, 1, 1, 0) & (1, 1, 1, 0, 1, 0, 1, 1, 1, 1) & (1, 0, 0, 1, 0, 1, 0, 1, 1, 0) \\
Triplet 2 & (0, 1, 1, 1, 1, 0, 1, 1, 0, 1) & (1, 0, 0, 1, 0, 1, 0, 1, 1, 0) & (0, 1, 0, 0, 0, 0, 0, 0, 0, 0) \\					
Triplet 3 & (0, 0, 0, 0, 1, 1, 1, 0, 1, 0) & (1, 1, 1, 1, 1, 1, 1, 1, 0, 1) & (0, 1, 1, 1, 1, 0, 1, 1, 0, 1) \\		
Triplet 4 & (0, 1, 1, 1, 1, 0, 1, 1, 0, 1) & (1, 1, 1, 1, 1, 1, 1, 1, 0, 1) & (0, 1, 0, 1, 1, 1, 1, 1, 0, 1) \\							
Triplet 5 & (0, 0, 0, 0, 1, 1, 1, 0, 1, 0) & (0, 1, 1, 0, 0, 1, 0, 0, 0, 1) & (0, 1, 0, 0, 0, 0, 0, 0, 0, 0) \\							
Triplet 6 & (1, 1, 1, 1, 1, 1, 1, 1, 0, 1) & (1, 1, 1, 0, 1, 1, 0, 1, 1, 0) & (0, 1, 1, 0, 0, 1, 0, 0, 0, 1) \\	
\bottomrule
\end{tabular}

\vspace{7pt}

\begin{tabular}{l|cccccccc}
\toprule
& Acc & BA & $F_1$ & $\kappa$ & CE & GM$_1$ & CC & SBA \\
\midrule
Acc & --- & 1 & 2 & 6 & 6 & 1 & 5 & 5 \\
BA & 1 & --- & 1 & 1 & 1 & 3 & 3 & 1 \\
$F_1$ & 2 & 1 & --- & 2 & 2 & 1 & 2 & 2 \\
$\kappa$ & 6 & 1 & 2 & --- & 4 & 1 & 3 & 3 \\
CE & 6 & 1 & 2 & 4 & --- & 1 & 3 & 3 \\
GM$_1$ & 1 & 3 & 1 & 1 & 1 & --- & 5 & 1 \\
CC & 5 & 3 & 2 & 3 & 3 & 5 & --- & 4 \\
SBA & 5 & 1 & 2 & 3 & 3 & 1 & 4 & --- \\
\bottomrule
\end{tabular}
\end{table}

\paragraph{Experiment within a weather forecasting service} In this section, we provide a detailed analysis of the precipitation prediction task discussed in Section~\ref{sec:exp_binary}. 

\begin{figure}
    \centering
    \includegraphics[width=0.48\textwidth]{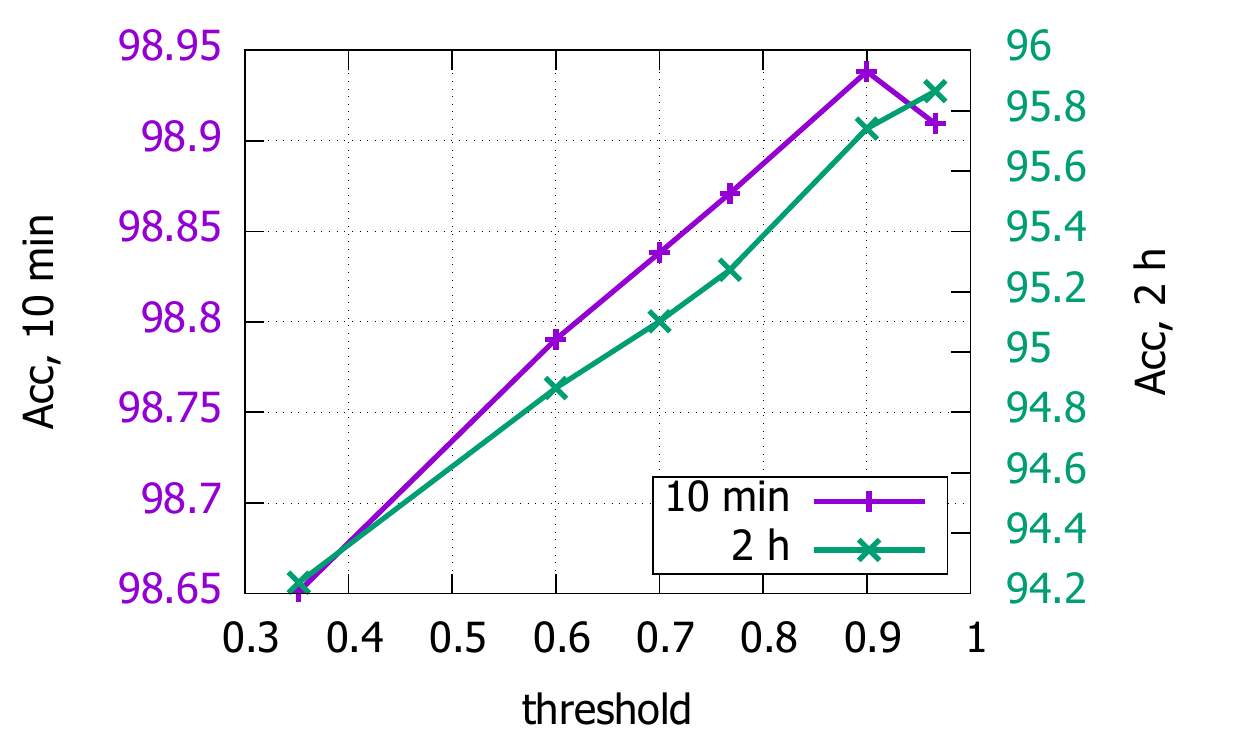}
    \includegraphics[width=0.48\textwidth]{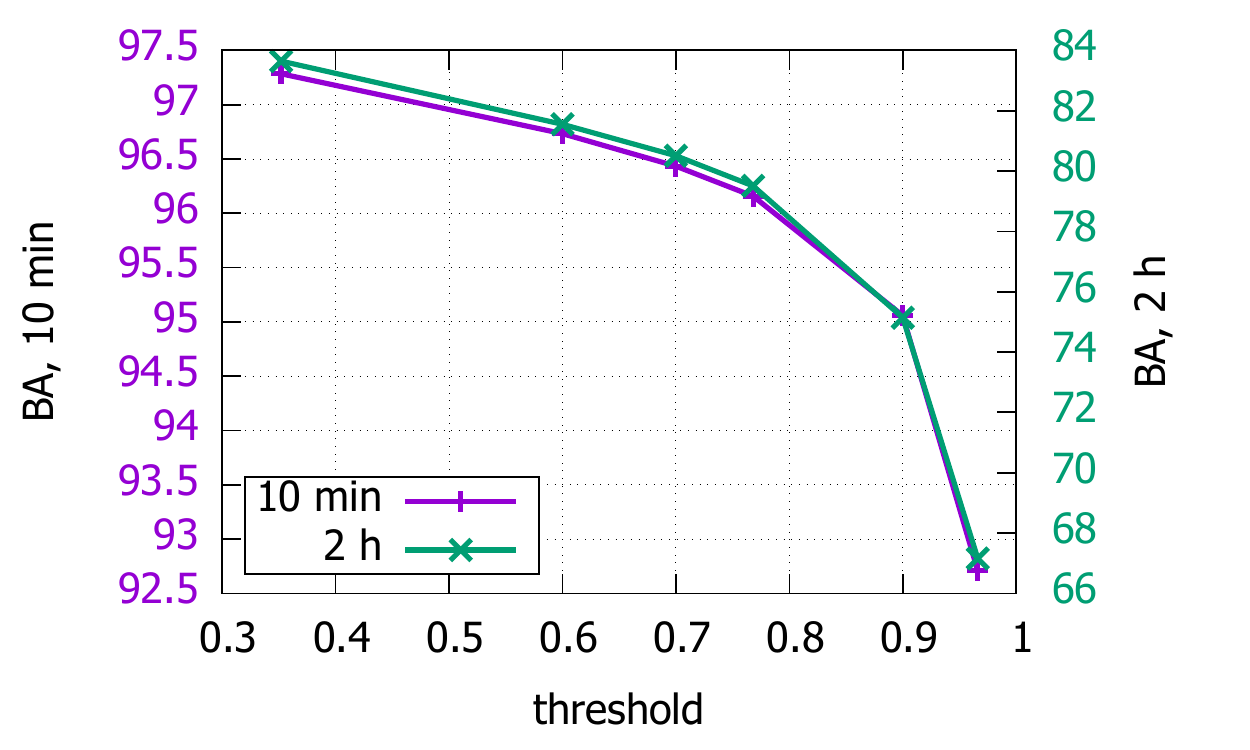}
    \includegraphics[width=0.48\textwidth]{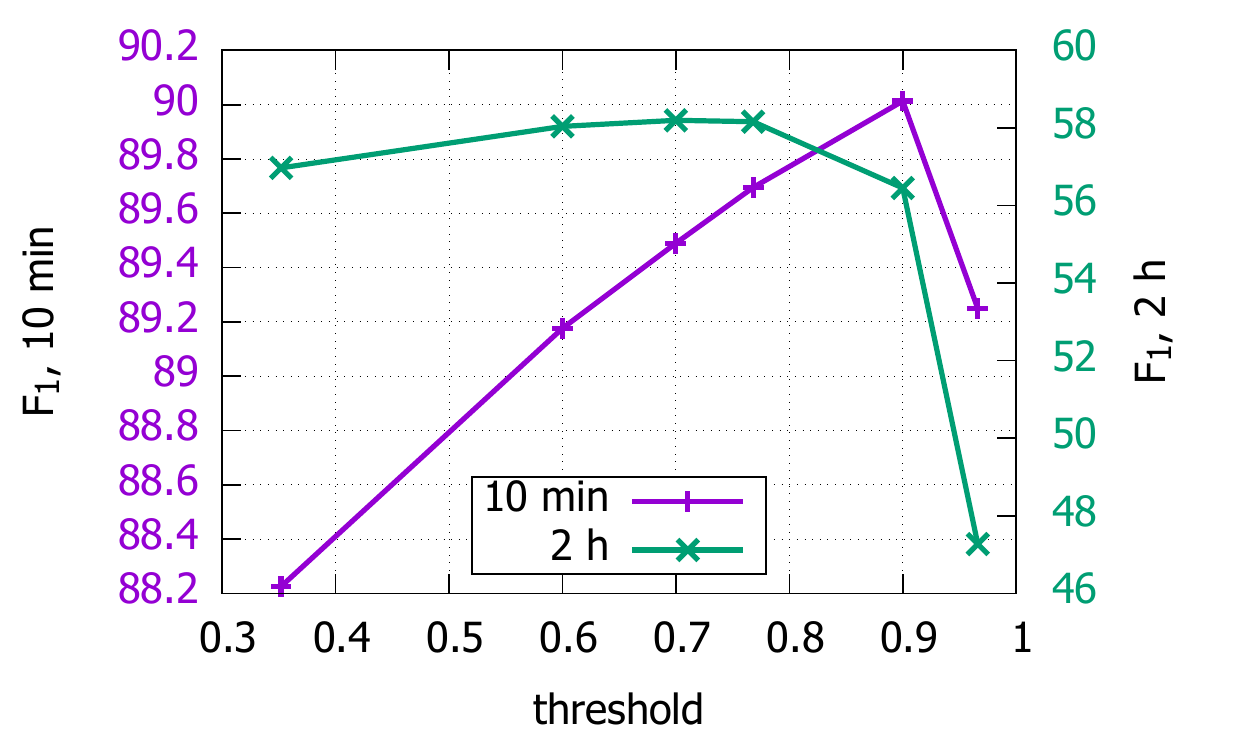}
    \includegraphics[width=0.48\textwidth]{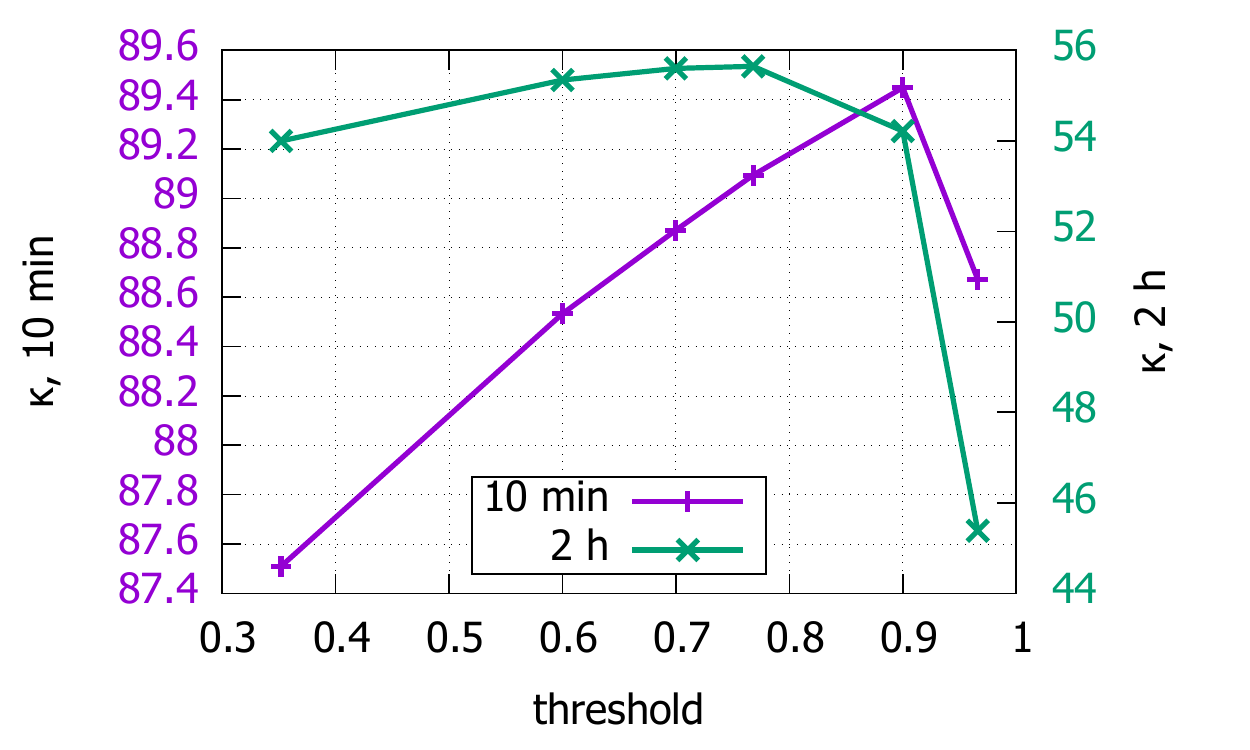}
    \includegraphics[width=0.48\textwidth]{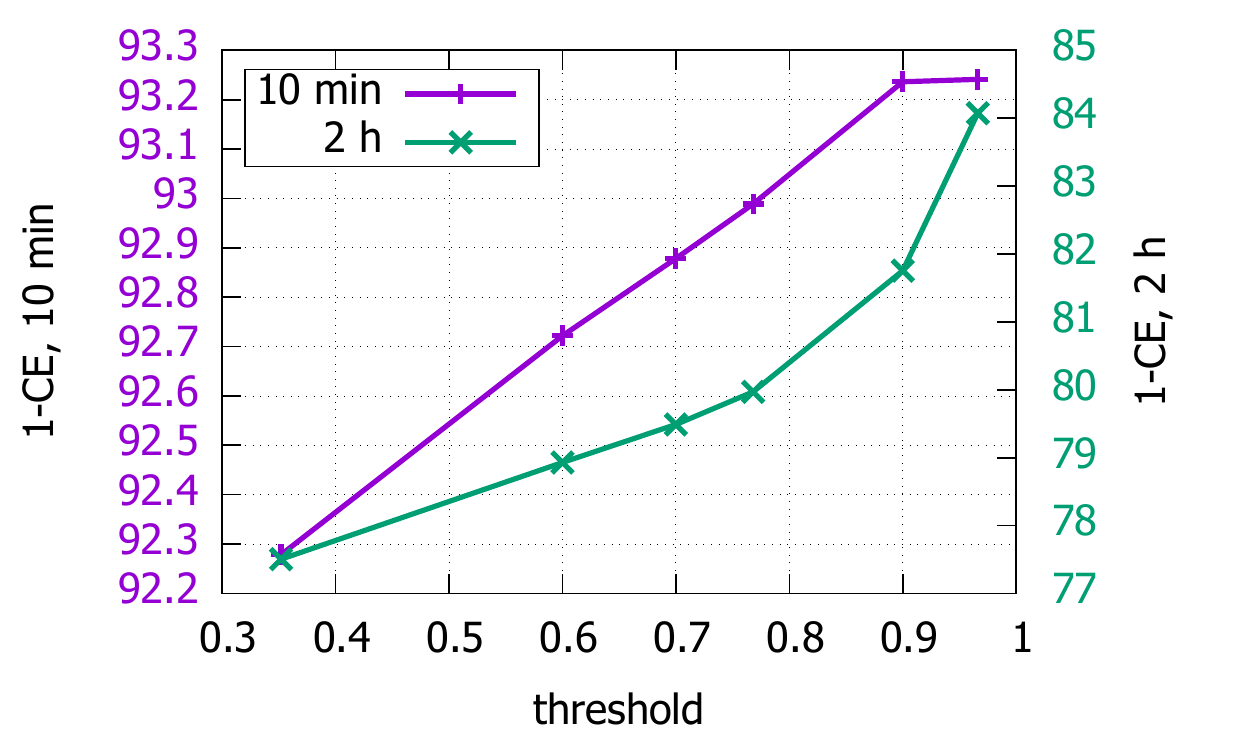}
    \includegraphics[width=0.48\textwidth]{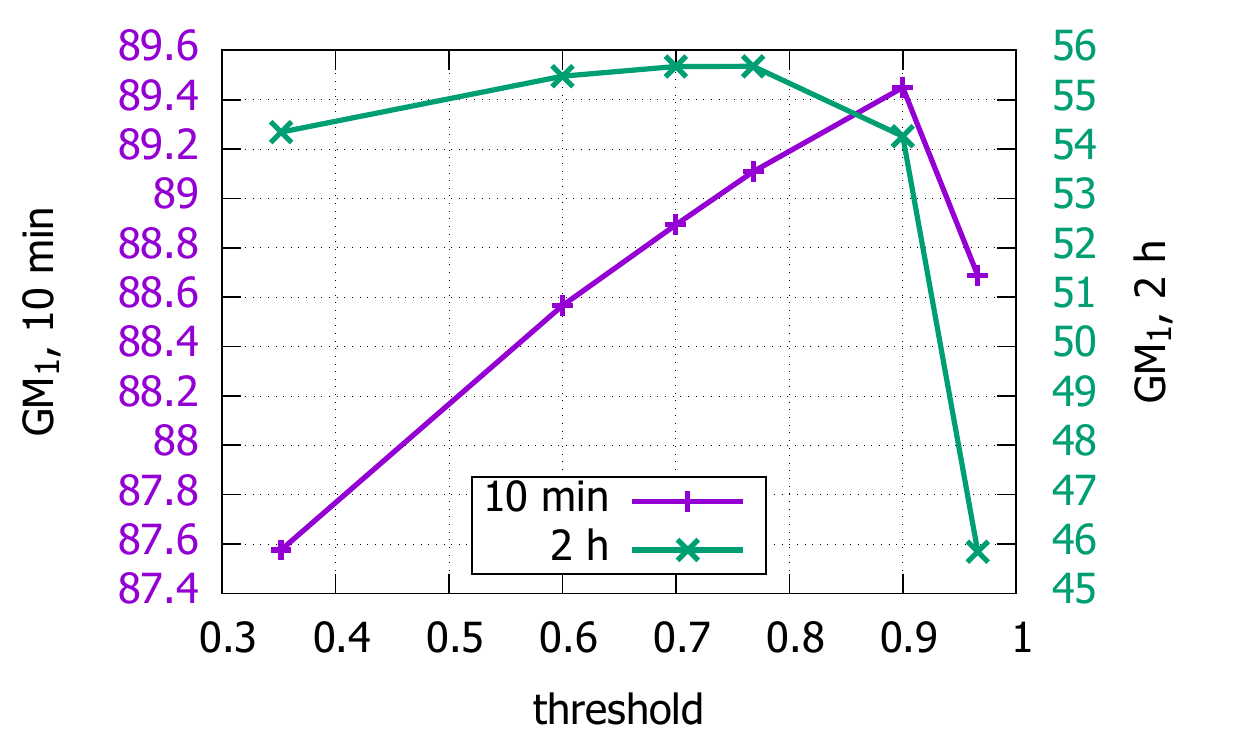}
    \includegraphics[width=0.48\textwidth]{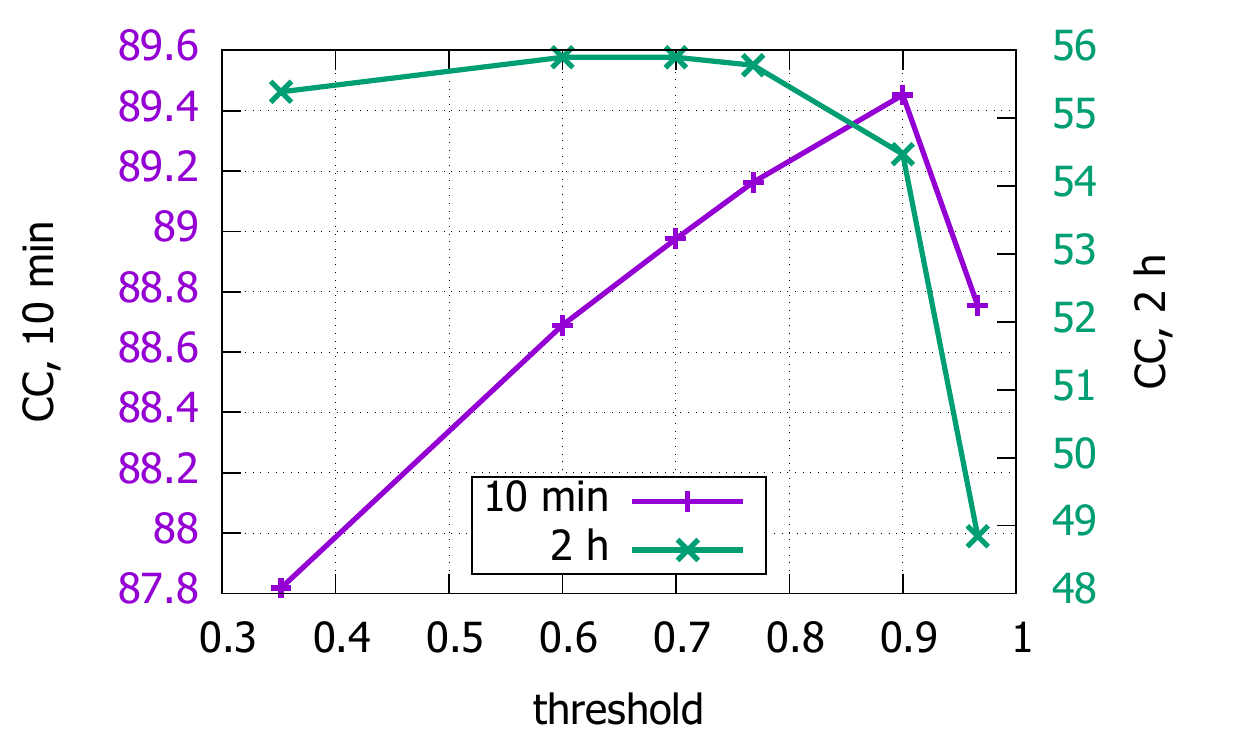}
    \includegraphics[width=0.48\textwidth]{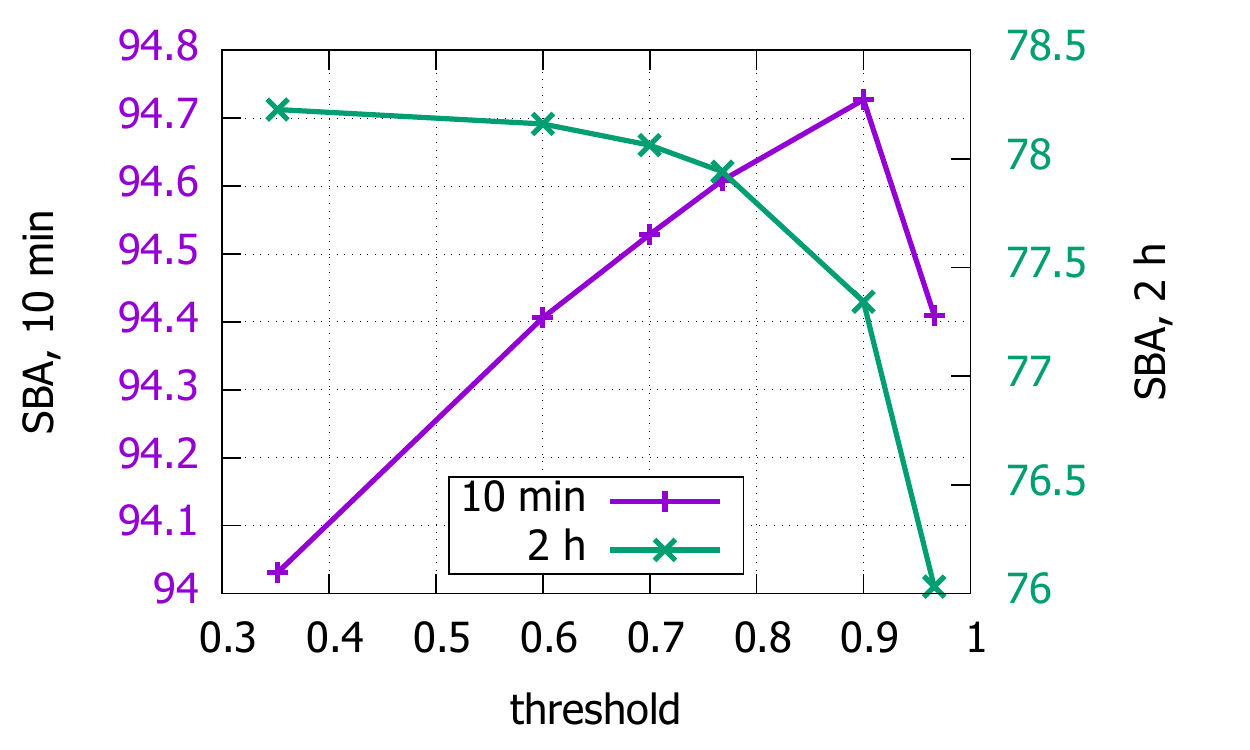}
    \caption{Dependence of measures on thresholds, for ten-minute and two-hour prediction horizons, the values are multiplied by 100}
    \label{fig:weather_thr}
\end{figure}

In Figure~\ref{fig:weather_thr}, we show the dependence of measures on the threshold that is used to convert soft predictions to binary labels. This is done separately for two prediction horizons: ten minutes and two hours. We make the following observations. For the ten-minute horizon, most of the measures agree that the optimal threshold is 0.9. However, Confusion Entropy favors the largest threshold, while Balanced Accuracy favors the smallest one. Interestingly, the behavior of measures significantly differs for the two-hour prediction interval. In this case, many of the measures favor either 0.6, 0.7, or 0.77. However, accuracy and CE prefer the largest threshold, while BA and SBA prefer the smallest one. Interestingly, this is the only experiment where we observe that SBA has such a noticeable disagreement with GM$_1$ and CC.

To better understand the differences between the measures, let us list average confusion matrices for the ten-minute and two-hour prediction horizons depending on a threshold (in increasing order). Here we show the relative values in percentages.

For ten minutes:
$$
\left(\begin{smallmatrix}
93.55 & 1.12 \\ 
0.22 & 5.11 \\
\end{smallmatrix}\right)
\hspace{10pt}
\left(\begin{smallmatrix}
93.76 & 0.91 \\ 
0.29 & 5.04 \\
\end{smallmatrix}\right)
\hspace{10pt}
\left(\begin{smallmatrix}
93.84 & 0.83 \\ 
0.33 & 5.01 \\
\end{smallmatrix}\right)
\hspace{10pt}
\left(\begin{smallmatrix}
93.91 & 0.76 \\ 
0.36 & 4.97 \\
\end{smallmatrix}\right)
\hspace{10pt}
\left(\begin{smallmatrix}
94.10 & 0.57 \\ 
0.49 & 4.85 \\
\end{smallmatrix}\right)
\hspace{10pt}
\left(\begin{smallmatrix}
94.33 & 0.34 \\ 
0.75 & 4.59 \\
\end{smallmatrix}\right)
$$

For two hours:
$$
\left(\begin{smallmatrix}
90.41 & 4.25 \\ 
1.47 & 3.87 \\
\end{smallmatrix}\right)
\hspace{10pt}
\left(\begin{smallmatrix}
91.32 & 3.34 \\ 
1.74 & 3.60 \\
\end{smallmatrix}\right)
\hspace{10pt}
\left(\begin{smallmatrix}
91.67 & 2.99 \\ 
1.87 & 3.47 \\
\end{smallmatrix}\right)
\hspace{10pt}
\left(\begin{smallmatrix}
91.96 & 2.70 \\ 
1.99 & 3.35 \\ 
\end{smallmatrix}\right)
\hspace{10pt}
\left(\begin{smallmatrix}
92.94 & 1.72 \\ 
2.51 & 2.83 \\ 
\end{smallmatrix}\right)
\hspace{10pt}
\left(\begin{smallmatrix}
93.98 & 0.68 \\ 
3.43 & 1.91 \\
 \end{smallmatrix}\right)
$$

Consider, for instance, the two smallest thresholds for the ten-minute horizon. It is easy to see that accuracy grows from 98.66\% to 98.80\%. In contrast, for Balanced Accuracy, the difference between the values can be written as:
\[
\Delta \text{BA} = \frac{\Delta c_{00}}{a_0} + \frac{\Delta c_{11}}{a_1} \approx \frac{0.21}{94.67} + \frac{-0.07}{5.33} < 0.
\]
So, Balanced Accuracy favors the smallest threshold. This can be explained by the fact that BA normalizes true positives ($c_{11}$) by a much smaller value, so that the impact of $c_{11}$ is much higher.

More interesting is the fact that for the ten-minute horizon, SBA agrees with most of the measures and strongly disagrees with BA. This can be explained by the fact that SBA also takes into account the distribution of predicted labels. For instance, for the two smallest thresholds, the difference becomes:
\[
\Delta \text{SBA} \approx \frac{0.21}{94.67} + \frac{-0.07}{5.33} + \left(\frac{93.76}{94.05} - \frac{93.55}{93.77}\right) + \left(\frac{5.04}{5.95} - \frac{5.11}{6.23}\right) > 0.
\]
Here the difference between the last two terms is positive and dominates all other differences. This happens because the false positive rate becomes significantly smaller. 

Tables~\ref{tab:weather_10min} and~\ref{tab:weather_2h} summarize inconsistency between different measures for the ten-minute and two-hour horizons. In particular, we can see that SBA and CC always agree for the ten-minute horizon, while they have almost 20\% disagreement for two hours.

\begin{table}
    \caption{Inconsistency of binary measures for rain prediction, horizon 10 minutes, \%}
    \label{tab:weather_10min}
    \vspace{4pt}
    \centering
    \begin{small}
    \begin{tabular}{l|rrrrrrrr}
    \toprule
 & Acc & BA\, & $F_1$\, & $\kappa$\,\,\, & CE\, & GM$_1$ & CC\, & SBA \\
\midrule
Acc & --- & 93.3 & 14.4 & 14.4 & 3.3 & 14.4 & 15.0 & 15.0 \\
BA & 93.3 & --- & 78.9 & 78.9 & 96.7 & 78.9 & 78.3 & 78.3 \\
$F_1$ & 14.4 & 78.9 & --- & 0.0 & 17.8 & 0.0 & 0.6 & 0.6 \\
$\kappa$ & 14.4 & 78.9 & 0.0 & --- & 17.8 & 0.0 & 0.6 & 0.6 \\
CE & 3.3 & 96.7 & 17.8 & 17.8 & --- & 17.8 & 18.3 & 18.3 \\
GM$_1$ & 14.4 & 78.9 & 0.0 & 0.0 & 17.8 & --- & 0.6 & 0.6 \\
CC & 15.0 & 78.3 & 0.6 & 0.6 & 18.3 & 0.6 & --- & 0.0 \\
SBA & 15.0 & 78.3 & 0.6 & 0.6 & 18.3 & 0.6 & 0.0 & --- \\
\bottomrule
\end{tabular}
\end{small}
\end{table}

\begin{table}
    \caption{Inconsistency of binary measures for rain prediction, horizon 2 hours, \%}
    \label{tab:weather_2h}
    \vspace{4pt}
    \centering
    \begin{small}
    \begin{tabular}{l|rrrrrrrr}
    \toprule
 & Acc & BA\, & $F_1$\, & $\kappa$\,\,\, & CE\, & GM$_1$ & CC\, & SBA \\
\midrule
Acc & --- & 98.3 & 63.3 & 58.3 & 1.7 & 61.1 & 72.2 & 91.7 \\
BA & 98.3 & --- & 35.0 & 39.4 & 100 & 37.2 & 25.6 & 6.1 \\
$F_1$ & 63.3 & 35.0 & --- & 4.4 & 65.0 & 2.2 & 8.9 & 28.3 \\
$\kappa$ & 58.3 & 39.4 & 4.4 & --- & 60.0 & 2.2 & 13.3 & 32.8 \\
CE & 1.7 & 100 & 65.0 & 60.0 & --- & 62.8 & 73.9 & 93.3 \\
GM$_1$ & 61.1 & 37.2 & 2.2 & 2.2 & 62.8 & --- & 11.1 & 30.6 \\
CC & 72.2 & 25.6 & 8.9 & 13.3 & 73.9 & 11.1 & --- & 18.9 \\
SBA & 91.7 & 6.1 & 28.3 & 32.8 & 93.3 & 30.6 & 18.9 & --- \\
\bottomrule
\end{tabular}
\end{small}
\end{table}

\subsection{Multiclass measures}\label{app:multiclass}

\paragraph{Image classification}

The extended results are shown in Table~\ref{tab:imagenet_ext}. The models are the following:\footnote{\href{https://github.com/rwightman/pytorch-image-models/blob/master/results/results-imagenet.csv}{https://github.com/rwightman/pytorch-image-models/blob/master/results/results-imagenet.csv}}
\begin{enumerate}
\item tf\_efficientnet\_l2\_ns
\item tf\_efficientnet\_l2\_ns\_475
\item swin\_large\_patch4\_window12\_384
\item tf\_efficientnet\_b7\_ns
\item tf\_efficientnet\_b6\_ns
\item swin\_base\_patch4\_window12\_384
\item swin\_large\_patch4\_window7\_224
\item dm\_nfnet\_f6
\item tf\_efficientnet\_b5\_ns
\item dm\_nfnet\_f5
\end{enumerate}
Note that the dataset is balanced, so accuracy coincides with BA, and weighted average coincides with macro average.

\begin{table}[t]
    \caption{Extended results for ImageNet, the values are multiplied by 100, inconsistencies are highlighted}
    \label{tab:imagenet_ext}
    \vspace{4pt}
    \centering
    \begin{small}
    \begin{tabular}{l|cccccccccccc}
\toprule
 & Acc/BA & $F_1$ & J & $\kappa$ & $1-$CE & GM$_1$ & CC & CC$^{mac}$ & SBA \\
\midrule
1 & 88.33 & 88.21 & 80.43 & 88.32 & 94.42 & 88.19 & 88.32 & 88.31 & 88.44 \\
2 & 88.23 & 88.08 & 80.25 & 88.21 & 94.38 & 88.07 & 88.22 & 88.20 & 88.35 \\
3 & 87.15 & 87.01 & 78.63 & 87.13 & 93.86 & 87.00 & 87.13 & 87.14 & 87.30 \\
4 & 86.83 & 86.64 & 78.08 & 86.82 & 93.64 & 86.63 & 86.82 & 86.78 & 86.95 \\
5 & 86.46 & 86.30 & 77.525 & 86.44 & 93.41 & 86.28 & 86.44 & 86.419 & 86.57 \\
6 & 86.43 & 86.27 & \textbf{77.531} & 86.42 & \textbf{93.51} & 86.26 & 86.42 & \textbf{86.423} & \textbf{86.60} \\
7 & 86.32 & 86.17 & 77.311 & 86.30 & 93.37 & 86.16 & 86.30 & 86.31 & 86.48 \\
8 & 86.31 & 86.12 & \textbf{77.314} & 86.29 & \textbf{93.41} & 86.10 & 86.30 & 86.28 & 86.47 \\
9 & 86.08 & 85.89 & 76.97 & 86.06 & 93.21 & 85.87 & 86.07 & 86.02 & 86.19 \\
10 & 85.72 & 85.55 & 76.51 & 85.70 & 93.05 & 85.53 & 85.70 & 85.70 & 85.89 \\
\bottomrule
    \end{tabular}
    \end{small}
\end{table}

\paragraph{Inconsistency for Yeast dataset}

In this experiment, we consider the Yeast dataset\footnote{\href{https://archive.ics.uci.edu/ml/datasets/Yeast}{https://archive.ics.uci.edu/ml/datasets/Yeast}} from the UCI repository~\cite{Dua:2019}. The task is to predict protein localization sites among ten possible variants. The class sizes are \{463, 429, 244, 163, 51, 44, 35, 30, 20, 5\}, so the dataset is highly unbalanced.

To this dataset, we apply the following algorithms from the scikit-learn library~\cite{algorithms}: DecisionTree, ExtraTree, ExtraTreesEnsemble, NearestNeighbors, RadiusNeighbors, RandomForest, BernoulliNB, GaussianNB, LabelSpreading, QuadraticDiscriminantAnalysis, LinearDiscriminantAnalysis, NearestCentroid, MLPClassifier, LogisticRegression, LogisticRegressionCV, RidgeClassifier, RidgeClassifierCV, LinearSVC. Thus, there are 18 algorithms giving 153 possible pairs. For each pair of measures and each pair of algorithms, we check whether the measures are consistent. Aggregating the results over all pairs of algorithms, we obtain Table~\ref{tab:yeast}.

We can see that for some measures the disagreement can be significant. For example, inconsistency is particularly high for Confusion Entropy, which does not satisfy most of the properties. Interestingly, the best agreement is achieved by CC and $\kappa$.

Finally, Table~\ref{tab:averagins} shows inconsistency of different averagings.

\begin{table}[t]
    \caption{Inconsistency of multiclass measures on the Yeast dataset, \%}
    \label{tab:yeast}
    \vspace{4pt}
    \centering
    \begin{small}
    \begin{tabular}{l|rrrrrrrrr}
    \toprule
 & Acc & BA\, & $F_1$\, & J\,\,\, & $\kappa$\,\,\, & CE\, & GM$_1$ & CC\, & SBA \\
\midrule
Acc & --- & 11.8 & 13.7 & 11.1 & 4.6 & 47.7 & 11.1 & 3.3 & 17.0 \\
BA & 11.8 & --- & 9.8 & 8.5 & 7.2 & 52.9 & 7.2 & 8.5 & 11.8 \\
$F_1$ & 13.7 & 9.8 & --- & 2.6 & 10.5 & 48.4 & 5.2 & 10.5 & 4.6 \\ 
J & 11.1 & 8.5 & 2.6 & --- & 9.2 & 49.7 & 6.5 & 9.2 & 7.2 \\
$\kappa$ & 4.6 & 7.2 & 10.5 & 9.2 & --- & 49.7 & 7.8 & 1.3 & 13.7 \\
CE & 47.7 & 52.9 & 48.4 & 49.7 & 49.7 & --- & 51.0 & 48.4 & 45.1 \\
GM$_1$ & 11.1 & 7.2 & 5.2 & 6.5 & 7.8 & 51.0 & --- & 7.8 & 8.5 \\
CC & 3.3 & 8.5 & 10.5 & 9.2 & 1.3 & 48.4 & 7.8 & --- & 13.7 \\
SBA & 17.0 & 11.8 & 4.6 & 7.2 & 13.7 & 45.1 & 8.5 & 13.7 & --- \\
\bottomrule
\end{tabular}
\end{small}
\end{table}

\begin{table}[]
\centering
\begin{small}
\caption{Inconsistency of averagings on the Yeast dataset}
\label{tab:averagins}
\vspace{4pt}
\centering
\begin{tabular}{l|ccc}
\toprule
 & $F_1^{mic}$ & $F_1^{mac}$ & $F_1^{wgt}$ \\
\midrule
$F_1^{mic}$ & --- & 13.73 & 3.27 \\
$F_1^{mac}$ & 13.73 & --- & 10.46 \\
$F_1^{wgt}$ & 3.27 & 10.46 & --- \\
\bottomrule
\end{tabular}
\hspace{8pt}
\vspace{5pt}
\begin{tabular}{l|ccc}
\toprule
 & J$^{mic}$ & J$^{mac}$ & J$^{wgt}$ \\
\midrule
J$^{mic}$ & --- & 11.11 & 2.61 \\
J$^{mac}$ & 11.11 & --- & 8.50 \\
J$^{wgt}$ & 2.61 & 8.50 & --- \\
\bottomrule
\end{tabular}
\vspace{5pt}
\begin{tabular}{l|cccc}
\toprule
 & CC & CC$^{mic}$ & CC$^{mac}$ & CC$^{wgt}$ \\
\midrule
CC & --- & 3.27 & 0.00 & 0.65 \\
CC$^{mic}$ & 3.27 & --- & 0.00 & 0.65 \\
CC$^{mac}$ & 0.00 & 0.00 & --- & 0.65 \\
CC$^{wgt}$ & 0.65 & 0.65 & 0.65 & --- \\
\bottomrule
\end{tabular}
\hspace{8pt}
\begin{tabular}{l|cccc}
\toprule
 & CD & CD$^{mic}$ & CD$^{mac}$ & CD$^{wgt}$ \\
\midrule
CD & --- & 3.27 & 0.00 & 0.65 \\
CD$^{mic}$ & 3.27 & --- & 0.00 & 0.65 \\
CD$^{mac}$ & 0.00 & 0.00 & --- & 0.65 \\
CD$^{wgt}$ & 0.65 & 0.65 & 0.65 & --- \\
\bottomrule
\end{tabular}
\vspace{5pt}
\begin{tabular}{l|ccc}
\toprule
 & GM$_1^{mic}$ & GM$_1^{mac}$ & GM$_1^{wgt}$ \\
\midrule
GM$_1^{mic}$ & --- & 11.76 & 7.19 \\
GM$_1^{mac}$ & 11.76 & --- & 7.19 \\
GM$_1^{wgt}$ & 7.19 & 7.19 & --- \\
\bottomrule
\end{tabular}
\end{small}
\end{table}


\end{document}